\documentclass[journal]{IEEEtran} 

\usepackage{cite}
\usepackage{url}
\usepackage{hyperref}
\usepackage{amsmath,amssymb,amsfonts}
\usepackage{algorithmic}
\usepackage{graphicx}
\usepackage{textcomp}
\let\labelindent\relax
\usepackage{enumitem}
\usepackage{multirow}
\usepackage{booktabs}
\usepackage{xcolor}
\usepackage{flushend}

\usepackage{nicefrac}
\usepackage{caption}
\usepackage{subcaption}
\usepackage{lipsum}
\renewcommand{\theenumi}{\roman{enumi}}
\usepackage[ruled,vlined,linesnumbered]{algorithm2e}
\usepackage{algorithmic,algorithm2e,float}

\SetKw{Continue}{continue}
\SetKw{Break}{break}

\usepackage{accents}

\newcommand{\vect}[1]{\boldsymbol{#1}}

\newcommand{\be}{\begin{equation}}
\newcommand{\ee}{\end{equation}}

\newcommand{\A}{\mathcal{A}}
\newcommand{\W}{\mathcal{W}}
\newcommand{\I}{\mathcal{I}}

\newcommand{\R}{\mathbb{R}}

\newcommand{\w}{\vect{w}}
\newcommand{\cvect}{\vect{c}}
\newcommand{\Cvec}{\vect{C}}
\newcommand{\vmu}{\vect{\mu}}

\newcommand{\T}{\mathcal{T}}
\newcommand{\Sset}{\mathcal{S}}

\newcommand{\KL}{\mathtt{KL}}

\mathchardef\mhyphen="2D
\newcommand{\AS}{\mathtt{AS}}
\newcommand{\baseline}{\mathtt{DC}}

\newcommand{\Unib}{\mathtt{Uni}}
\newcommand{\Uni}{\mathtt{Uni\mhyphen A}}

\newcommand{\ie}{\textit{i.e.,}~}
\newcommand{\eg}{\textit{e.g.,}~}

\newcommand{\grey}[1]{\textcolor{gray}{#1}}

\usepackage{amsthm}
\theoremstyle{definition}
\newtheorem{problem}{Problem}
\newtheorem{theorem}{Theorem}

\newtheorem{assumption}{Assumption}
\newtheorem{definition}{Definition}

\newtheorem*{remark*}{Remark}
\newtheorem{claim}{Claim}

\newtheorem{lemma}{Lemma}
\newenvironment{subproof}[1][\proofname]{%
  \begin{proof}[#1]%
}{%
  \end{proof}%
}

\begin{document}
\title{\LARGE \bf
{Statistically Distinct Plans for Multi-Objective Task Assignment}
}

\author{Nils Wilde, \textit{Member}, IEEE, and Javier Alonso-Mora, \textit{Senior Member}, IEEE
\thanks{This research is supported by the European Union's Horizon 2020 research and innovation program under Grant 101017008.}
\thanks{The authors are with the Department for Cognitive Robotics, 3ME,
Delft University of Technology, Delft, Netherlands, 
\texttt{\{n.wilde, j.alonsomora\}@tudelft.nl}}%
}

\markboth{IEEE TRANSACTIONS ON ROBOTICS}%
{{Statistically Distinct Plans for Multi-Objective Planning and Task Assignment}}
\maketitle

\begin{abstract}
We study the problem of finding statistically distinct plans for stochastic planning and task assignment problems such as online multi-robot pickup and delivery (MRPD) when facing multiple competing objectives.
In many real-world settings robot fleets do not only need to fulfil delivery requests, but also have to consider auxiliary objectives such as energy efficiency or avoiding human-centered work spaces. 
We pose MRPD as a multi-objective optimization problem where the goal is to find MRPD policies that yield different trade-offs between given objectives. 
There are two main challenges: 1) MRPD is computationally hard, which limits the number of trade-offs that can reasonably be computed, and 2) due to the random task arrivals, one needs to consider statistical variance of the objective values in addition to the average.
We present an adaptive sampling algorithm that finds a set of policies which i) are approximately optimal, ii) approximate the set of all optimal solutions, and iii) are statistically distinguishable. We prove completeness and adapt a state-of-the-art MRPD solver to the multi-objective setting for three example objectives.
In a series of simulation experiments we demonstrate the advantages of the proposed method compared to baseline approaches and show its robustness in a sensitivity analysis. The approach is general and could be adapted to other multi-objective task assignment and planning problems under uncertainty.

\end{abstract}


\section{Introduction}

Autonomous robots are becoming increasingly capable of solving complex tasks in challenging and dynamically changing environments. These advancements will soon enable the large scale deployment of robot fleets in a wide range of applications including transportation, on-site assistance service, autonomous mobility on demand, environmental monitoring and inspection.
For instance, the deployment of mobile robots in hospitals and care homes for assistive tasks such as material transport promises to reduce the workload of perpetually overburdened skilled personnel \cite{niechwiadowicz2008robot, abubakar2020arna}. 

{Many robot planning problems such as path planning, multi-robot task assignment (MRTA), multi-agent path finding (MAPF) or multi-robot pickup and delivery (MRPD) need to consider multiple competing objectives simultaneously.}
Usually, the primary goal is to provide the optimal quality of service (QoS), captured by measures such as the average or maximum wait times, delivery delays, system throughput, or the number of on-time deliveries, among others.
However, in practice, the deployment of autonomous robots may require the consideration of additional objectives.
For instance, service robots might need to balance between quality of service (QoS) and operation cost \cite{cap2018multi}, or consider sustainable costs \cite{niranjani2022minimization}.
Moreover, when navigating in human-centered workspaces robot fleets need to consider established social norms,  such as avoiding areas of the environment with high foot traffic \cite{wilde2020improving} or social navigation objectives \cite{biswas2022socnavbench}.

In this paper, we {study how we can compute statistically distinct system plans when optimizing for multiple objectives under uncertainty. We specifically focus on multi-objective MRPD \cite{camisa2022multi}, a special case of MRTA where a fleet of robots needs to service dynamically appearing transportation requests. However, the proposed solution technique could be extended to multi-objective MAPF or other MRTA variants when considering stochastic task arrivals or travel times.}

\begin{figure}[t]
    \centering
   \includegraphics[width=0.99\linewidth]{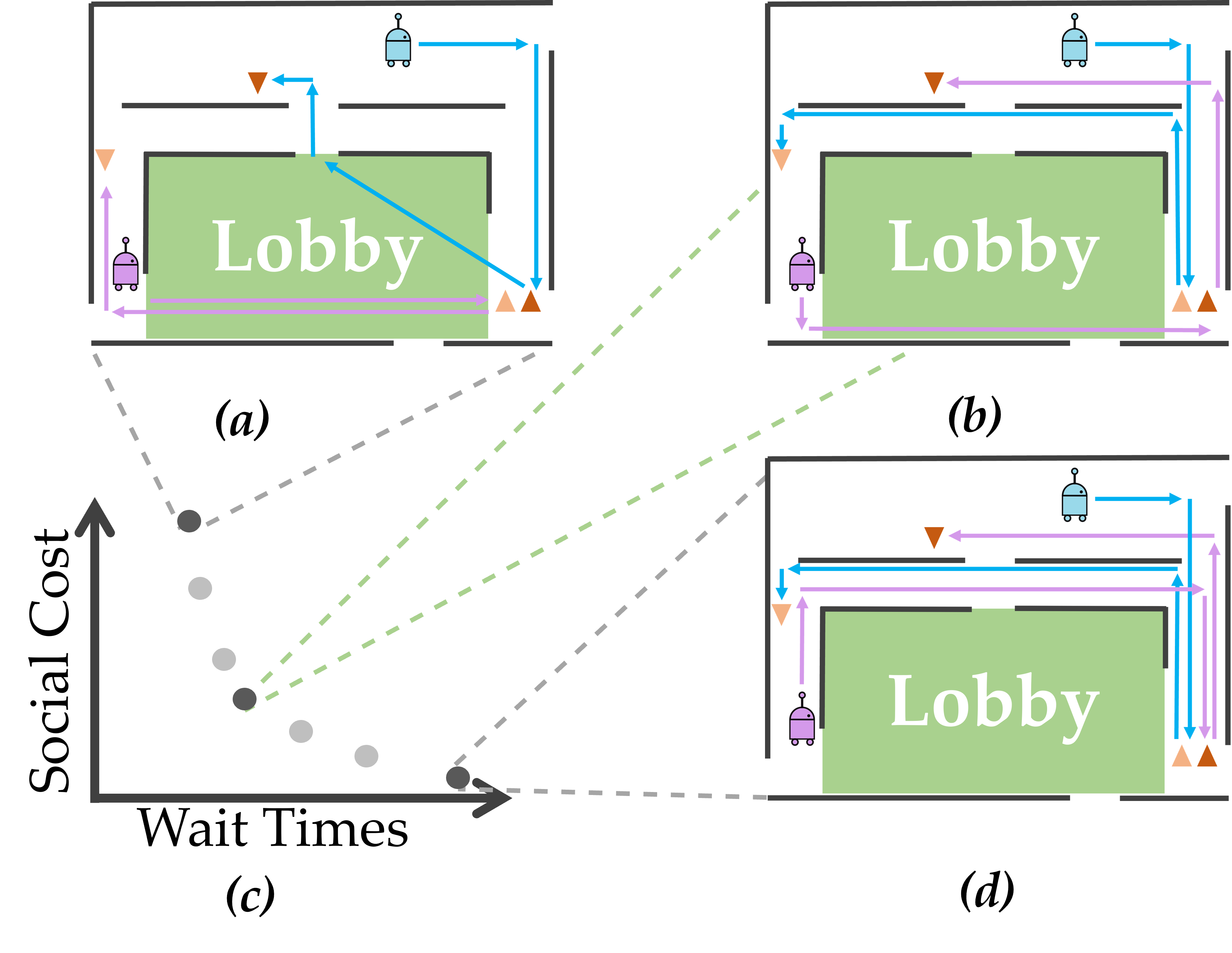}    
    \caption{Schematic example for multi-objective multi-robot pickup and delivery.
    Two robots are required to deliver two packages (dark and light orange) from a pickup location (triangle up) to a delivery point (triangle down).
    Different system plans vary in their trade-off between task wait times and a social cost for traversing a lobby with high foot traffic.
    The plot in (c) shows the values of the two objectives functions attained by the different system plans, called the \textit{Pareto front} of the multi-objective optimization problem.
    }
    \label{fig:Intro_example}
\end{figure}
Figure \ref{fig:Intro_example} shows an example for MRPD with two competing objectives: minimizing wait time of deliveries and avoiding a lobby area where robot traffic is undesired.
The system plan shown in (a) prioritizes QoS, while (d) avoids robot traffic in the lobby whenever possible. 
Arguably, both solutions are not ideal. On one hand, only optimizing QoS completely ignores the social aspects. On the other hand, when robots try to avoid human-centered areas at all cost, the wait times become very high, as illustrated in (c). When tasks have deadlines, the fleet might then not be able to deliver all packages on time. Thus, an important challenge for deploying MRPD systems is balancing between different objectives, \textit{i.e.,} finding intermediate trade-offs. For instance, (b) shows a system where the lobby is only traversed once, allowing the purple robot to still deliver its delivery on time.

Which trade-off is most appropriate often depends on various stakeholders such as the system operator and the people present in the environment.
In this paper, we study how we can find a set of MRPD plans \textit{i.e.,} policies, that lead to different trade-offs between such competing objectives. 
This gives the users of an MRPD system a palette of options to choose a solution from that fits their individual preferences.

We pose MRPD as a multi-objective optimization (MOO) problem, and seek to find policies with \textit{Pareto}-optimal trade-offs, \ie policies where neither objective can be improved without impairing another objective.

This problem bears two unique challenges: 1) due to computational hardness, it is usually not feasible to compute optimal solutions for MRPD. Thus, the computed trade-offs are not necessarily Pareto-optimal, but only heuristic solutions. 
Moreover, even heuristic solutions are often computationally burdensome to obtain. Therefore the number of calls to an MRPD solver should be minimized, making it impractical to generate a large ground set of MRPD policies and then selecting a representative subset based on the obtained trade-offs. 
2) in \emph{online} MRPD, the requests' arrival times as well as pickup and delivery locations are usually following a stochastic process. 
Thus, only exploring trade-offs for one specific sequence of requests is of limited interest, since this would represent only a certain time window (\eg day) of the deployment of an MRPD system. Arguably, it is more relevant to find trade-offs between the different objectives that lead to similar behaviour across different time windows, \ie different realizations of the stochastic process. Thus, each policy is associated with a \textit{distribution} of cost trade-offs.

We formulate the problem of finding a set of MRPD policies such that the expected values of their cost distributions represent the Pareto-front when optimizing for the expected costs. However, the cost distributions should also remain statistically significantly different from one another such that each policy retains unique characteristics, \ie is a distinct option for operating the MRPD system.
We use the popular approach of applying a linear scalarization to convert the MOO problem into a single objective function constituted by a weighted sum of the competing objectives. Each possible choice of scalarization weights then corresponds to an MPRD policy. However, picking weights that lead to a desirable set of policies is not trivial, since the relationship between weights and MRPD costs is often non-linear \cite{botros2022error, branke2008multiobjective}. Picking regularly spaced weights can lead to several policies having similar behaviours, \ie trade-offs, while other possible system behaviours are not covered by any policy.
Therefore, we propose an adaptive sampling algorithm to find different scalarization weights. Our approach minimizes the \textit{dispersion} of means of the distribution of cost trade-offs. Further, when adding new policies, we consider the variance of the MRPD costs to avoid choosing policies that are statistically indistinguishable.

\subsection{Contributions}
The main contributions of our work are as follows: 
\begin{enumerate}
    \item We pose the problem of stochastic multi-objective MRPD. To reflect the stochastic nature of task arrivals, the formulation considers statistical variance of the objectives in addition to their mean values.
    
    \item We propose an adaptive sampling algorithm to find MRPD policies that approximate the \textit{expected} Pareto-front and ensures that MRPD plans are statistically \textit{distinguishable}.
    
    \item We establish completeness of the proposed algorithm.
    
    \item We provide example MRPD configurations for three different objective functions building on a state-of-the-art algorithm for single-objective MRPD.

    \item In simulation experiments, we showcase our approach for two different MRPD scenarios with varying numbers of tasks and robots and demonstrate its advantages compared to several baseline approaches. 

    
\end{enumerate}

\subsection{Related Work}

Many robotic planning problems face the challenge of being required to simultaneously optimize multiple objectives, for instance in path and trajectory planning \cite{ren2022multi, botros2022error,wilde2020improving}, autonomous driving \cite{tunable_traj_planner_urban, karkusdiffstack, botros2022tunable, zuo2020mpc} transportation and
mobility on demand \cite{cap2018multi}, multi-robot planning \cite{ren2021multi, cai2021probabilistic, ren2021multiInterval, schillinger2017multi, Wilde2021LearningSubmod}.

Designing robotic systems that face multiple-objectives requires representations of Pareto-fronts, which is a well-known problem in optimization \cite{zitzler1999multiobjective, branke2008multiobjective, schutze2020pareto, teich2001pareto, pereyra2013equispaced}, but also studied for specific robotics applications \cite{saberifar2022charting, lee2018sampling, botros2022error, cap2018multi,lavalle1998optimal, parisi2017manifold}. 
A common approach to multi-objective optimization is linear scalarization, \ie using the weighted sum of the individual objective functions to pose a single optimization problem \cite{branke2008multiobjective}. The set of Pareto-optimal solutions is then approximated by exploring different scalarization weights \cite{zeng2021simultaneous, botros2022error,xu2021multi, wilde2020improving, kent2020human}.
However, finding useful weights is often challenging \cite{botros2022error,branke2008multiobjective,mores2023multi, bhonsale2022ellipsoid}. 

\paragraph{Multi-objective multi-robot problems}
Finding trade-offs between objectives is relevant in many multi-robot problems such as multi-robot task assignment (MRTA) \cite{sadeghi2017heterogeneous, luo2014provably, ray2021multi, wilde2022Online}, dynamic vehicle routing (DVR) \cite{bullo2011dynamic, rios2021recent, psaraftis2016dynamic, botros2023optimizing}, multi-robot pickup and delivery (MRPD) \cite{bai2022group, camisa2022multi, ritzinger2016survey}, automated mobility-on-demand (AMoD) systems \cite{alonso2017demand, simonetto2019real}, and multi-agent path finding (MAPF), among others.

The authors of \cite{cap2018multi} proposed a method for exploring different Pareto-optimal trade-offs for customer service and operation cost for AMoD systems. 
Our work considers a general multi-objective MRPD formulation without being constrained to specific objective functions. Similar to \cite{cap2018multi} we also explore different system plans by finding scalarization weights. However, the work in \cite{cap2018multi} is based on uniform sampling and average values, while we propose an adaptive sampling method for stochastic multi-objective optimization problems and demonstrate its advantages over uniform sampling.
In the context of MRTA, the work of \cite{wei2020particle} considers the problem of balancing the total team cost and workload balance by simultaneously minimizing the average and maximum cost of robot tours. 
The authors of \cite{tolmidis2013multi} pose an MRTA problem that considers task completion times in conjunction with energy consumption of the individual robots as well as task priorities.
Multi-objective trade-offs are also considered in multi-agent path finding (MAPF) \cite{ren2021multi, ren2021multiInterval}.
In MAPF, the objective is to find \emph{collision free} paths for a fleet of agents, \ie agents are not allowed to be at the same location at any given time. In contrast, our MRPD formulation does not consider inter-agent collisions, but rather focuses on a pickup and delivery requests arriving online.
Moreover, multi-objective task specification in temporal logic for multi-robot systems were studied in \cite{schillinger2017multi, cai2021probabilistic}.
Finally, a recurring problem in multi-robot systems is finding the optimal fleet for a given set of tasks \cite{zhao2022graph} as well as fleet sizing \cite{chandarana2018determining, chaikovskaia2021sizing, cap2018multi, rjeb2021sizing}.
Solutions to these problems also solve a multi-objective optimization problem as they balance between the capability of the fleet and its acquisition and operational cost.

In summary, existing work on multi-objective multi-robot systems usually focuses on specific objectives. In contrast, our method is not tailored to certain objective functions and additionally considers the statistical variance in costs due to stochastic task arrivals.

\paragraph{User preferences for competing objectives}
Researchers in human-robot interaction (HRI) study the problem of \emph{reward learning} which seeks to interactively learn a reward or cost function that best describes a user's preference for robot behaviour \cite{jeon2020reward}. Usually, this reward function is modelled as a weighted sum of features  \cite{abbeel2004apprenticeship, hadfield2016cooperative, sadigh2017active, biyik2019asking, Wilde2021LearningSubmod,  wilde2020improving}, which is equivalent to the linear scalarization of a MOO problem. Thus, learning a user's reward function is equivalent to finding the Pareto-optimal trade-off that best fit their preferences.
{Furthermore, effectively computing a representative set of Pareto-optimal solutions improves how well system behaviour can be adapted to user preferences \cite{botros2022error}.}
In our earlier work, we studied the problem of learning user preferences for material transport with consideration of task efficiency and following social norms \cite{wilde2019bayesian, wilde2020improving}. However, these works considered a set of individual start-goal transportation tasks. In contrast, this paper focuses on online MRPD where we take a fleet of robots and multiple trips per vehicle into account with stochastic task arrivals.

\paragraph{Approximating Pareto-fronts}
Given the wide-spread applications of multi-objective optimization, several fundamental techniques for computing Pareto-fronts have been studied over the years \cite{zitzler1999multiobjective, branke2008multiobjective, schutze2020pareto, teich2001pareto, pereyra2013equispaced}.
However, popular approaches such as gradient descent methods, evolutionary algorithms \cite{branke2008multiobjective} or random walks \cite{lee2018sampling} assume that objective values can be easily obtained and thus make use of frequently evaluating the objectives for different parameters. Computing MRPD solutions is computationally burdensome even when using heuristic solutions, making such approaches impractical.

Closely related to our work, the work presented in \cite{mores2023multi, bhonsale2022ellipsoid} considers MOO under uncertain parameters, and pose a Pareto-approximation problem where samples are required to be statistically significantly different, {focusing on applications in chemical engineering}. 
{Similar to our work, this approach iteratively places new samples on the Pareto-front using a divide-and-conquer (DC) approach to find new weights. Both algorithms stop dividing when solutions are no longer significantly different. However, a key difference is that  \cite{mores2023multi, bhonsale2022ellipsoid} chooses new Pareto-samples by uniformly placing weights (similar to a breadth-first-search).
In our work, we place new weights such that we greedily minimize dispersion, \textit{i.e.,} the distance in the objective space. 
This makes our method more sample efficient and results in a more homogeneous coverage of the Pareto-front when the number of samples is limited.
}%
Moreover, we provide a theoretical analysis establishing completeness of our approach.
We compare both methods in simulations.

The authors of \cite{tesch2013expensive} studied the MOO for robotics when objectives are expensive to evaluate. Their method replaces fitness functions used in GA with an expected improvement in hypervolume. However, this requires a surrogate objective function but no principal approach is given in \cite{tesch2013expensive}. In contrast, our work is based on a greedy placement of new samples to reduce the \textit{dispersion}, a measure for the distance between points on the Pareto-front. This can be directly approximated from the objective values that have been already sampled.
Our earlier work \cite{botros2022error} studies the problem for finding a Pareto-approximation with bounded regret for general multi-objective problems formulated as weighted sums. The number of samples and thus computations of objective values is budgeted, however, a limiting assumption is that an optimal solution can be obtained for any given weight. This makes the solution from \cite{botros2022error} unsuitable for multi-objective MRPD since computing exact solutions for MRPD is computationally prohibitively expensive for most practically relevant instances.
Similar to our work, the authors of \cite{kim2006adaptive} propose an adaptive weighted sum (AWS) method for finding weights in order to approximate the Pareto-front. The AWS method iteratively identifies patches of the Pareto-front that require additional samples. Each patch is subsequently refined by adding constraints to the optimization problem and solving it again. Unfortunately, such constraints cannot be directly incorporated in an online MRPD formulation. Thus, our method does not rely on constraints but adds new weights guided by \textit{dispersion}, a measure for the largest gaps in the current approximation of the Pareto-front. We iteratively add new sample solutions by selecting new weights as the midpoint of existing weights where difference in costs is largest and then optimizing for the new weights.

Lastly, in this paper we consider that the MOO problem has stochastic objective values since some inputs might be random variables, \textit{e.g.,} the transportation requests appear following a stochastic process. 
The authors of \cite{teich2001pareto} study uncertain objectives where costs cannot be computed exactly. For the case that attainable objective values fall into known bounds, they propose a notion of \textit{probabilistic dominance}. Unfortunately, tight bounds are usually not available in MRPD a-priori.

In summary, the unique challenges of the multi-objective MRPD problem studied in this paper are the computational hardness and expense of obtaining individual sample solutions, which makes most state-of-the-art Pareto-approximation techniques impractical, and the uncertainty in the problem inputs, requiring samples of the Pareto-front to be selected such that resulting solutions are statistically different.

\section{Problem Formulation}
\subsection{Preliminaries}
{\textbf{Notation:}
Vectors are denoted with bold symbols ($\w$) and we use subscript indices to identify its elements ($w_i$). Upper-case letters denote sets $(S)$, where we identify elements with a superscript index ($s^i$ or $\w^i$).
}

{\textbf{Multi-objective optimization:}}
Consider a multi-objective optimization problem (MOOP) \cite{branke2008multiobjective} where the domain is some vector space $\mathcal{X}$. We want to find a solution $\vect{x}\in \mathcal{X}$ that simultaneously minimizes $n$ different functions, \ie  that solves $\min_{\vect{x}}\{f_1(\vect{x}), \dots, f_n(\vect{x})\}$. In general, the solution to a MOOP is not a unique element $\vect{x}$, but a set of \textit{Pareto-optimal} solutions. 
We briefly review the definitions of \textit{dominated solutions} and the Pareto-front.
\begin{definition}[Dominated solution]
    Given a MOOP and two solutions $\vect{x},\vect{x}'\in\mathcal{X}$, vector $\vect{x}$ \textit{dominates} $\vect{x}'$ when $f_i(\vect{x})\leq f_i(\vect{x}')$ holds for all $i=1,\dots,n$ and there exists a $j\in\{1,\dots, n\}$ where $f_j(\vect{x})< f_j(\vect{x}')$. This is denoted by $\vect{x}\prec\vect{x}'$. 
\end{definition}

\begin{definition}[Pareto Front]
    Given a MOOP, the set of \textit{Pareto-optimal} solutions is the subset of all solutions that are not \textit{dominated} by another solution. This set is referred to as the \textit{Pareto-front}.
\end{definition}

\subsection{Online multi-robot pickup and delivery}

We now revisit the standard MRPD problem where tasks require robots to pickup items in some location in the workspace and then transport them to a different location.

The robot environment is encoded in a weighted graph $G=(V,E,d)$ where $V$ and $E$ are vertices and edges, and weights $d$ describe the duration of traversing an edge.
Let $R=\{r_1,\dots,r_m\}$ be a fleet of $m$ robots that have to serve a set of $n$ pickup and delivery tasks $\T=\{T_1, \dots, T_n\}$.
Each task is a tuple $T=(s,g, t^r, t^d)$ where $s$ and $g$ are vertices in $V$, representing a pickup and drop-off location, $t^r$ is the release time when the task is requested and $t^d>t^r$ is a deadline.
Let $t^a(r)$ be the time a robot $r$ arrives at the pickup vertex $s$, and let $t^f(r)$ be the time robot $r$ arrives at the drop-off vertex $g$.
A task is serviced successfully when $s$ is visited before $g$ and $t^f(r)\leq t^d$. 
Further, let $L_r$ be the set of tasks loaded by robot $r$.
All robots have a capacity $\kappa$, \ie $|L_r|\leq \kappa$ must hold at all times.
Upon visiting a pickup vertex $s$ the respective task $T$ is added to $L_r$ if $|L_r|<\kappa$, when visiting a drop-off $g$ it is removed%
\footnote{Note: To avoid unintentional loading and unloading, we can design the graph $G$ such that $s$ and $g$ are copies of an existing vertex in $V$.}.
We assume that tasks appear online following a stochastic process $\mathcal{Y}$, and their pickup and drop-off locations are sampled randomly from a distribution over the vertices in the graph.

MRPD constitutes two subproblems: which robot serves which task, and what route each robot takes.
\paragraph{Routing Problem}
For some robot $r$, let $v$ denote its current location.
To serve tasks $\T$, the robot needs to find a tour $\tau$ that starts at $v$ and services all tasks in $\T$. 
Thus, the routing problem solves 
\be
\label{eq:routing}
\min_{\tau} \gamma(\T, \tau),
\ee
where $\gamma(\T, \tau)$ is some non-negative cost function. For instance, this can evaluate if the tasks were delivered on time, or the duration between request and delivery time.

\paragraph{Assignment Problem}
An assignment is a set $\A \subseteq \{(r_i, T_j)| r_i\in R, \, T_j\in\T \}$ such that for every $T_j\in \T$ there exists exactly one pair in $\A$ containing $T_j$, \ie every task is assigned to exactly one robot.
However, a robot can be assigned to multiple tasks; thus, each $r_i$ can appear in multiple pairs in $\A$. Finally, let $\T_i(\A)$ be the set of tasks assigned to robot $r_i$ under $\A$.
{The goal is to find an assignment of tasks to robots, as well as tours for each robot, such that the cost $\gamma$ is minimized for all tasks. The assignment problem is formulated as}
\be
{
\begin{aligned}
\min_{\A} &\sum_{r_i\in R\;} 
\min_{\tau_i}\gamma(\T_i(\A), \tau_i)\\
s.t.\,& \tau_i \text{ serves all tasks }\T_i(\A),\\
&\T_1(\A)\cup\dots\cup \T_m(\A) = \T.
\end{aligned}
}
\label{eq:MRPD_obj}
\ee
{The nested optimization can be solved in a two-stage coupled approach: First, optimal tours for potential pairing of groups of tasks and robots are computed. Based on these tours, a group of tasks is assigned to each robots \cite{alonso2017demand}.}

{In an \textit{offline} problem all tasks are known before robot deployment such that an assignment of tasks and routes for all robots are computed \textit{offline} and then executed. However, in many practical problems not all tasks are known initially: further tasks might be requested while other tasks are already being serviced. Such \textit{online} settings require frequently adding new tasks to the current assignment and replanning routes to optimally accommodate new requests. Further, our formulation does not consider inter-agent collision. Instead, we assume that collisions are avoided by a low-level controller and the average travel times are abstracted by the cost of edges on the graph.
}
The task arrival can be modelled with a random process $\mathcal{Y}$, making the set of tasks $\T$ is a partially observed random variable. An optimal assignment $\A$ is found by a \emph{policy} $\pi$ that recomputes the current assignment and routes periodically as new tasks arrive.
Thus, we redefine the cost over a policy $\pi$ and tasks $\T$, denoted by $c(\pi, \T)$.

\subsection{Multi-objective MRPD}

In practice, the performance of an MRPD system might be evaluated by a user with different objectives in mind. 
For instance, when operating in a human-centered workspace, users might have preferences for robot navigation based on several, potentially conflicting objectives.

Thus, we consider a multi-objective optimization (MOO) formulation for MRPD with bounded, positive cost functions $c_1(\pi, \T), \dots, c_n(\pi, \T)$ replacing the objective of \eqref{eq:MRPD_obj} with
\be
\begin{aligned}
\min_{\pi} &
\{
c_1(\pi, \T), c_2(\pi, \T), \dots, c_n(\pi, \T)
\}.
\end{aligned}
\label{eq:MRPD_MOO}
\ee

We refer to this problem as multi-objective MRPD, abbreviated as MO-MRPD.
In this paper, we are interested in exploring possible Pareto-optimal trade-offs to help system operators to efficiently deploy robot fleets.
Since the task arrivals are stochastic, the costs $c_1(\pi, \T), \dots, c_n(\pi, \T)$ are random variables, which we collect in a vector $\cvect(\pi, \T)=[c_1(\pi, \T)\, \dots\, c_n(\pi, \T)]$. Further, let $\vect{\mu}(\pi)$ be the vector containing the expected values of $\cvect(\pi, \T)$, \textit{i.e.,} $\mu_i(\pi) = \mathbb{E}_{\T}[c_i(\pi, \T)]$.
Using the mean costs, we introduce the notion of an \textit{expected Pareto front}.

\begin{definition}[Expected Pareto Front]
Given cost functions $c_1(\pi,\T),\dots, c_n(\pi,\T)$, the \textit{expected Pareto-front} {$\mathcal{P}(\T)$} is the set of Pareto-optimal solutions to the MOOP $\min_{\pi}
\{
\mu_1(\pi, \T), \mu_2(\pi, \T), \dots, \mu_n(\pi, \T)
\}$.
\end{definition}

To formalize the goal of our problem, we define dispersion as a measure of distance between points on the expected Pareto-front.
\begin{definition}[Dispersion]
\label{def:dispersion}
Given an MO-MRPD instance with cost functions $c_1(\pi,\T),\dots, c_n(\pi,\T)$, let $\Pi=\{\pi^1, \dots, \pi^k\}$ be a collection of policies. 
The dispersion of $D(\Pi)$ is the maximum distance between a point $\vect{p}$ on the expected Pareto-front $\mathcal{P}(\T)$ and the closest mean cost vector $\vmu(\pi^i, \T)$ for any $i=1,\dots, k$. {In detail,
\be
D(\Pi) = \max_{\vect{p}\in \mathcal{P}(\T)} \min_{\pi \in \Pi} ||\vmu(\pi,\T) - \vect{p}||_2.
\ee
}
\end{definition}

{
Lastly, we consider two different policies $\pi^i$ and $\pi^j$ to be \textit{statistically distinct} if their corresponding multi-variate distributions $\cvect(\pi^i, \T)$ and $\cvect(\pi^j, \T)$ are statistically significantly different. This can be captured with common statistical measures such as hypothesis tests \cite{thomas2006elements}.}

\begin{figure}[t]
    \centering
   \includegraphics[width=0.7\linewidth]{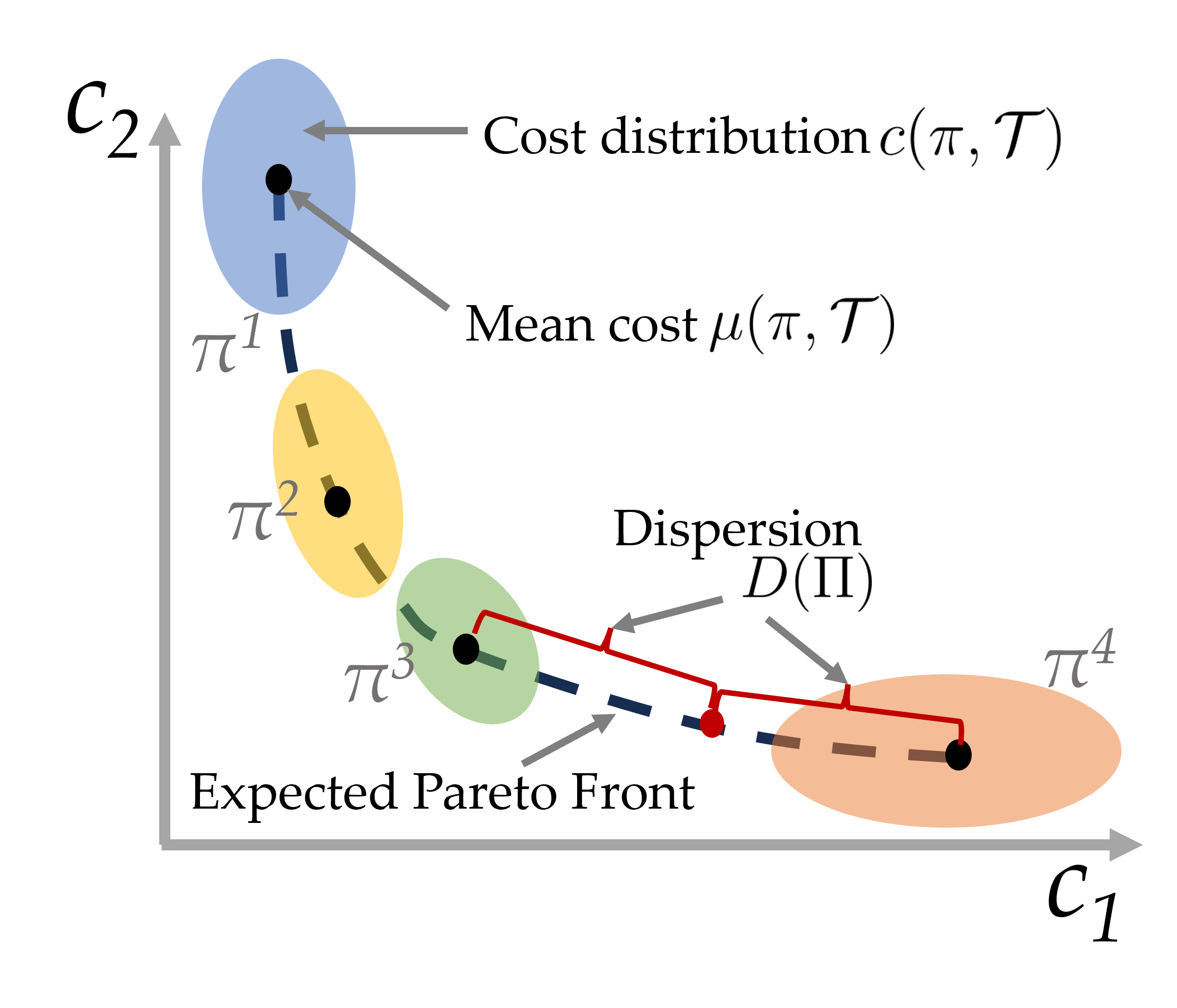}    
    \caption{{Illustration of cost distributions, mean cost, expected Pareto-front, and dispersion for four different policies $\pi_1,\dots,\pi_4$ solving a problem with two objectives.}
    }
    \label{fig:dispersion}
\end{figure}
\subsection{Problem Statement}
Building on the the preliminary concepts, we formally state the problem of approximating the set of optimal solutions to an MO-MRPD problem.

\begin{problem}[Approximating MO-MRPD]
\label{prob:MOMRPD}
Given a graph $G$, a fleet of $m$ robots, a stochastic process $\mathcal{{Y}}$ to generate a sequence of tasks $\T$, some cost functions $c_1(\pi,\T), \dots, c_n(\pi,\T)$ and some integer $K>0$, find a set of policies $\Pi=\{\pi^1, \dots, \pi^k\}$ where $k\leq K$ such that

\begin{enumerate}[label=(\roman*)]
    \item For any policy $\pi\in \Pi$ and a realization of the task sequence $\T$ the cost vector $\cvect(\pi, \T)$ is Pareto-optimal.
    \item The \textit{dispersion} $D(\Pi)$ is minimized.
    \item Any two policies $\pi^i, \pi^j \in \Pi$ represent different behaviours, \ie are statistically distinct.
\end{enumerate}

\end{problem}

Solving Problem \ref{prob:MOMRPD} provides a system operator with sampled options for fleet behaviour, that are optimal, approximate any optimal behaviour that is not part of the samples, and are distinguishable from one another.
%

\subsection{Generalization}
{Problem \ref{prob:MOMRPD} can be generalized to include a broader range of planning problems. Consider any robotic planning problem that i) optimizes for multiple objective functions, and ii) has one or multiple random variables as its input, \textit{e.g.,} random task arrivals or task locations, random service times or random travel times and random rewards. Exploring the solutions space for such problems can then be formulated as in Problem \ref{prob:MOMRPD}, \textit{i.e.,} finding solutions that are Pareto-optimal, represent the expected Pareto-front and are statistically distinct with respect to the randomness of the inputs.
Thus, our solution technique presented in the following section can be adapted to problems such as Dynamic Vehicle Routing and MRTA with stochastic task arrivals, single-robot path planning or MAPF in dynamic environments, the stochastic Canadian traveller problem, or orienteering and informative path planning with stochastic travel times or rewards.
For the remainder of this paper we will focus on MO-MRPD.}

\section{Approach}
\label{sec:approach}

We begin with characterizing the computational hardness of Problem \ref{prob:MOMRPD} and then present our approach to approximating MO-MRPD solutions. In essence, we cast the multi-objective problem into a scalarized single objective optimization, where the cost functions are traded-off with weights. We then solve Problem \ref{prob:MOMRPD} by adaptively selecting scalarization weights and thus MRPD policies that result in a set of Pareto-optimal trade-offs. 
The proposed method does not assume specific cost functions $c_1(\pi, \T), \dots, c_n(\pi, \T)$, it only requires them to be bounded.
For a case study, we show how three specific cost functions relevant in MRPD problems can be incorporated in an MRPD solver in Section \ref{sec:example_costs}.

\subsection{Hardness Results}
Problem \ref{prob:MOMRPD} is related to the multi-objective shortest path (MOSP) problem which is known to be NP-hard \cite{ehrgott2005multicriteria}.
In general, single-objective MRPD is a variation of dynamic vehicle routing and is already NP-hard for commonly used cost functions such as minimizing delivery time. Yet, MO-MRPD is also computationally intractable even for a single robot and only one task (in which case single-objective MRPD would be trivial).
\begin{lemma}[Hardness result]
\label{lem:hardness}
    MO-MRPD with a single robot and one task is NP-hard.
\end{lemma}
\begin{proof}
    This is shown by a reduction from MOSP to MO-MRPD. A MOSP instance is constituted by a graph $G=(V,E)$, start and goal vertices $s,g \in V$ and some cost functions $\gamma_1,\dots, \gamma_n$ assigning costs to edges \cite{ehrgott2005multicriteria}. Its decision version answers if there exists a path $P$ between a given start and goal vertices $s,g \in V$ such that $\sum_{e\in E(P)}\gamma_i(e)\leq \alpha$ for all $i$ and some constant $\alpha$.
    We convert this into an input of an MO-MRPD instance where  a single robot is initially located at $s$, and one task task requiring the robot to go from $s$ to $g$ before an arbitrarily late deadline. The MRPD cost functions are then the MOSP cost functions $\gamma_1,\dots, \gamma_n$. 
    For a sufficiently large budget $K$ the solution to the MO-MRPD instance is a set of policies that correspond to all Pareto-optimal paths from $s$ to $g$. This trivially answers the MOSP instance.  
\end{proof}

In summary, Lemma \ref{lem:hardness} shows that MO-MRPD is hard even if there is only a single task such that the assignment part of the problem becomes trivial and the routing part is reduced to finding a path between two vertices. Thus, the computational challenge of MO-MRPD does not only stem from complexity of the underlying single-objective MRPD, but finding a set of policies that correspond to multi-objective solutions as described in Problem \ref{prob:MOMRPD} is itself another intractable problem.

\subsection{Scalarization of MO-MRPD}

A common approach to solve multi-objective optimization (MOO) problems such as \eqref{eq:MRPD_MOO} is using scalarization to obtain a single-objective function. The most common form is linear scalarization where the objectives are combined in a weighted sum:
\be
\label{eq:MRPD_LSMOOP}
\min_{\pi} 
\underbrace{w_1 c_1(\pi(\w), \T)+\dots + w_n c_n(\pi(\w), \T)}_{\w \cdot\cvect(\pi(\w), \T)}.
\ee
The weight vector $\w=[w_1, \dots, w_n]$ describes a trade-off between the different MRPD objectives and thus is an input to the policy $\pi(\w)$ for solving the routing and assignment problem. Hence, using a lineaer scalarization, Problem \ref{prob:MOMRPD} becomes one of finding a set of weights $\Omega=\{\w^1, \dots, \w^K\}$.
Without loss of generality, we assume that $\w$ lies in the set 
\be
\W=\{\w\in\mathbb{R}^n_{\geq0}\ | \sum_i w_i =1 \},
\ee 
which we refer to as the weight space.
Further, we assume that we have access to a policy $\pi(\w)$ that solves the MRPD problem for the scalarized cost function in \eqref{eq:MRPD_LSMOOP} for given weights $\w$. In the next Section \ref{sec:example_costs}, we will adapt a state-of-the-art MRPD solver to obtain such a policy for three exemplary cost functions.

\subsection{Pareto approximation via weight sampling}
\label{sec:algorithm}

We propose an algorithm that finds a set of scalarization weights $\Omega=\{\w^1, \dots, \w^K\}$
such that the corresponding policies $\pi^1, \dots, \pi^K$ solve Problem \ref{prob:MOMRPD}.

A commonly used approach to find different trade-offs of cost functions is sampling weights \emph{uniformly}. However, the corresponding solutions are often not placed uniformly on the Pareto front since the mapping from weights to the objective values is non-linear \cite{botros2022error}. This can lead to several weights yielding similar objective values.
Thus, we propose an adaptive strategy that greedily minimizes the dispersion.\\

\paragraph{Algorithm Description}
\begin{algorithm}[t]	
	\DontPrintSemicolon 
	\KwIn{A graph $G$, robot fleet $R$, a stochastic process for task arrivals $\mathcal{Y}$, MRPD cost functions $c_1, \dots, c_n$, MRPD policy $\pi(\w)$, sampling budget $K$, \# training instances $\eta$, overlap threshold $\Delta$}
	\KwOut{Sets of weights $\Omega$ and expected costs $\Gamma$.}
	 $\I\leftarrow $ Set of $\eta$ random MRPD instances drawn from $\mathcal{Y}$\\
    $\Omega\leftarrow\;\{\vect{e}^1,\dots, \vect{e}^n\}$ \grey{// Standard basis weights}\\
    
    $\Gamma\leftarrow\;\{\mathtt{MRPD}(\w=\vect{e}^1, \I),\dots, \mathtt{MRPD}(\w=\vect{e}^n, \I)$\} \\
    $\Sset\leftarrow\;\{(\vect{e}^1,\dots, \vect{e}^n)\}$ \grey{// Initial set with basis simplex}\\
	\For{$k=n$ to $K$} {
        $\w'\leftarrow \mathtt{Find\_New\_Weight}(\Sset, \Gamma, \Delta)$\\
        $\Cvec_{\I}'\leftarrow\mathtt{MRPD}(\w', \I) $ \grey{ // Compute MRPD cost}\\
        $\Omega, \Gamma,\Sset \leftarrow \mathtt{Update}(\Omega,\Gamma,\Sset, \w', \Cvec_{\I}', \Delta)$\\

    }

	\Return{$\Omega$, $\Gamma$}
	\caption{Adaptive sampling of MOO-MRPD}
	\label{alg:Pareto_Explore}
\end{algorithm}

Algorithm \ref{alg:Pareto_Explore} provides a high-level overview of the proposed approach. {After a detailed description of its components, we discuss a bi-objective example, including an illustration in Figure \ref{fig:alg_example}.}
We maintain a collection of subsets -- in particular \textit{simplexes} -- of the weight set $\W$ from which we iteratively sample new weights and then partition the simplexes further, similar to a \textit{bisection} algorithm in higher dimensions.
The algorithm uses the sub-routine $\mathtt{MRPD}(\w, \I)$ that solves the scalarized problem from equation \eqref{eq:MRPD_LSMOOP} for given weights $\w$ and $\eta$ different task sequences 
$\I=\{\T_1,\dots,\T_{\eta}\}$. 
The function $\mathtt{MRPD}(\w, \I)$ returns the set of cost vectors $C_{\I}(\w)=\{\cvect(\pi(\w), \T_1),\dots,\cvect(\pi(\w), \T_{\eta})\} $. {That is, we compute $\eta$ different cost vectors, each of dimension $n$ to approximate the $n$-dimensional distribution of $\cvect(\w,\T)$.} 

In detail, the algorithm begins by sampling $\eta$ different realizations of the task arrival process $\I=\{\T_1,\dots,\T_{\eta}\}$ (line 1).
Let $\vect{e}^1,\dots, \vect{e}^n$ denote the vectors of the standard basis in $\mathbb{R}^n$, which correspond to solving only the single-objective MRPD problems. First, we add these vectors to the set of sampled solutions $\Omega$ and compute their costs (line 2, 3). The tuple $(\vect{e}^1,\dots, \vect{e}^n)$ is saved in a list of \textit{$n$-simplexes} (line 4).
In the main loop we iteratively find a new candidate sample $\w'$ and compute the cost distribution $\Cvec_{\I}(\w')$ (lines 6 and 7). 
We then update the set of $n$-simplexes $\Sset$ using the new sample $\w'$ (line 8), detailed in Algorithm \ref{alg:update_segments}. 
Finally, the algorithm stops when $K$ sample solutions have been computed.

Next, we will provide details on the two core components of the Algorithm, \textit{i.e.,} how we identify the most promising new weight $\w'$ (line 6), and how we update the simplexes (line 8).

\paragraph{Ensuring statistical difference}
An important characteristic of the Algorithm is that it does not add a new weight $\w'$ when its cost distribution is too similar to the distribution of an existing sample. Thus, functions $\mathtt{Find\_New\_Weight}$ (line 6) and $\mathtt{Update}$ (line 8, and Algorithm \ref{alg:update_segments}), conduct a test for statistical significance.

To that end, let the function $h(\w^i, \w^j)$ 
evaluate the probability of error for a \textit{hypothesis test} between sampled costs $\Cvec_{\I}(\w^i)$ and $\Cvec_{\I}(\w^j)$ \cite{thomas2006elements}.
That is, we compute how likely it is that a statistical test would wrongly conclude that the samples for $\Cvec_{\I}(\w^i)$ correspond to the policy $\pi(\w^j)$. 
Let $\KL(\w^i||\w^j)$ be the \textit{Kullback-Leibler} divergence between $C_{\I}(\w^i)$ and $C_{\I}(\w^j)$.
The probability of error is then derived from the \textit{Chernov-Stein Lemma} \cite{thomas2006elements} as
\be
h(\w^i,\w^j) = e^{-\KL(\w^i||\w^j)}.
\label{eq:hyp_error}
\ee

\paragraph{Finding the next sample}
We now specify the function $\mathtt{Find\_New\_Weight}$ from line 6. Let $\w^i,\w^j$ belong to the same simplex $s$; we refer to the unordered pair $(\w^i,\w^j)$ as an edge of $s$.
The idea is to greedily reduce the dispersion between samples. 
Without having access to the set of Pareto-optimal solutions (which we are trying to approximate), we cannot directly evaluate the dispersion as introduced in Definition \ref{def:dispersion}.
Thus, we use an auxiliary measure: 
Given a simplex $s=(\w^1,\dots, \w^n)$ {let $\vmu^i$ denote the mean cost of the policy optimizing for weight $\w^i$}. We define the \textit{pair-wise dispersion} as $d^{ij} = ||\vmu^i-\vmu^j||$ for all $i,j=1,\dots, k$ and $i\neq j$.

A strong candidate for a new weight is then the midpoint of an edge $(\w^i,\w^j)$ where the pairwise dispersion is largest.
However, the midpoint $\w'$ of edge $(\w^i,\w^j)$ might not yield a new solution, but instead result in the same costs as either $\w^i$ or $\w^j$, \textit{i.e.,} $\vmu'=\vmu^i$ might hold.
To ensure convergence, we introduce a discount factor $\alpha((\w^i,\w^j))$. 
Given an edge $(\w^i,\w^j)$ with mean costs $(\vmu^i,\vmu^j)$, the factor counts how often $\mathtt{Find\_New\_Weight}$ previously returned a weight $\w'$ that was the midpoint of some edge $(\w^l,\w^p)$ with the same mean costs, \textit{i.e.,} where $(\vmu^i,\vmu^j)= (\vmu^l,\vmu^p)$.
Further, let $H_{\Delta}(\w^i,\w^j)$ be a binary variable describing the outcome of the statistics test of distributions $\Cvec_{\I}(\w^i)$ and $\Cvec_{\I}(\w^j)$ with respect to some threshold $\Delta$. We use the convention that $1$ describes the case when $h(\w^i,\w^j)\leq \Delta$ and thus the distributions are sufficiently different.
The \textit{discounted pair-wise dispersion} is then defined as 
\be
\label{eq:disc_disp}
D((\w^i,\w^j))=
\frac{H_{\Delta}(\w^i,\w^j)}{2^{\alpha((\w^i,\w^j))}}
||\vmu^i-\vmu^j||
.
\ee
The function $\mathtt{Find\_New\_Weight}$ selects the edge $e'$ among all simplexes $s\in\Sset$ that maximizes $D(e')$ and returns its midpoint $\w'$.
In Section \ref{sec:theoretical} we further discuss the necessity of the discount factor as part of our proof of convergence.\\

\paragraph{Update Function}
\begin{algorithm}[t]	
	\DontPrintSemicolon 
	\KwIn{Set of weights $\Omega$ and solutions $\Gamma$, set of simplexes $\Sset$, new weight $\w'$, samples from cost distribution $\vect{C}(\w')$, threshold $\Delta$}
	\KwOut{Updated samples, solutions and simplexes, $\Omega$, $\Gamma$, $\Sset$}

    
    \For{$s$ in $\Sset$ where $\w'$ lies in the convex hull of $s$}{
    $\Sset \leftarrow \Sset\setminus\{s \}$\\
    Find weights $\w^i, \w^j\in s$ where $\w'$ is the midpoint\\
    $s^{i} = s\setminus\{{\w^i}\} \cup \{\w'\}$
    \grey{ // Swap in new weight} \\
    $s^{j} = s\setminus\{{\w^j}\} \cup \{\w'\}$
    \grey{ // Swap in new weight} \\
    $ \Sset\leftarrow \Sset\cup \{s^{i},s^{j}\} $ \grey{ // Add new simplexes} 
  
    }
    
    
        \If{$h(\w', {\w})\leq \Delta$ for all $\w$ in $\Omega$}
        {
             $\Omega\leftarrow \Omega \cup \{\w'\}$ \grey{ // Add new sample}\\
        $\Gamma\leftarrow \Gamma \cup \{\Cvec_{\I}'\}$ \grey{ // Save cost distribution}\\
        }

	\Return{$\Omega$, $\Gamma$, $\Sset$}
	\caption{Update simplexes}
	\label{alg:update_segments}
\end{algorithm}

Algorithm \ref{alg:update_segments} updates the set of simplexes as well as sampled weights and solutions.
Since we chose $\w'$ to be the midpoint of weights of a simplex, it lies on a simplex's boundary and thus may be inside more than one simplex. Thus, we split every simplex $s$ containing $\w'$ into smaller simplexes.

To split a simplex $s$, we identify the weights $\w^i$ and $\w^j$ in $s$ for which the new weight $\w'$ is the midpoint (line 3).
We then create two new simplexes by  individually substituting $\w'$ for $\w^i$ and $\w^j$ (lines 4-6).
The new sample $\w'$ is only added to the set of samples $\Omega$ (and its corresponding solution $\Cvec'_{\mathcal{I}}$ to the set $\Gamma$) when it passes the statistics test with respect to all previously sampled solution (lines 7-9).
We illustrate the update function in Figure \ref{fig:alg_example_splits} for an example with three objectives. The initial simplex is defined by the three basis weights (since all weights lie in $\W$, \textit{i.e.,} their elements sum to $1$, we can project the three dimensional case into 2D). In the first iteration, a new sample $\w'$ is placed on the bottom edge, and the simplex is split. In the second iteration, $\mathtt{Find\_New\_Weight}(\Sset, \Gamma)$ may decide to place a new sample on the central edge. In that case Algorithm \ref{alg:update_segments} splits both simplexes $s^2$ and $s^3$.

\begin{figure}[t]
    \centering
    \begin{subfigure}[t]{0.15\textwidth}
            \centering
             \includegraphics[width=0.99\linewidth]{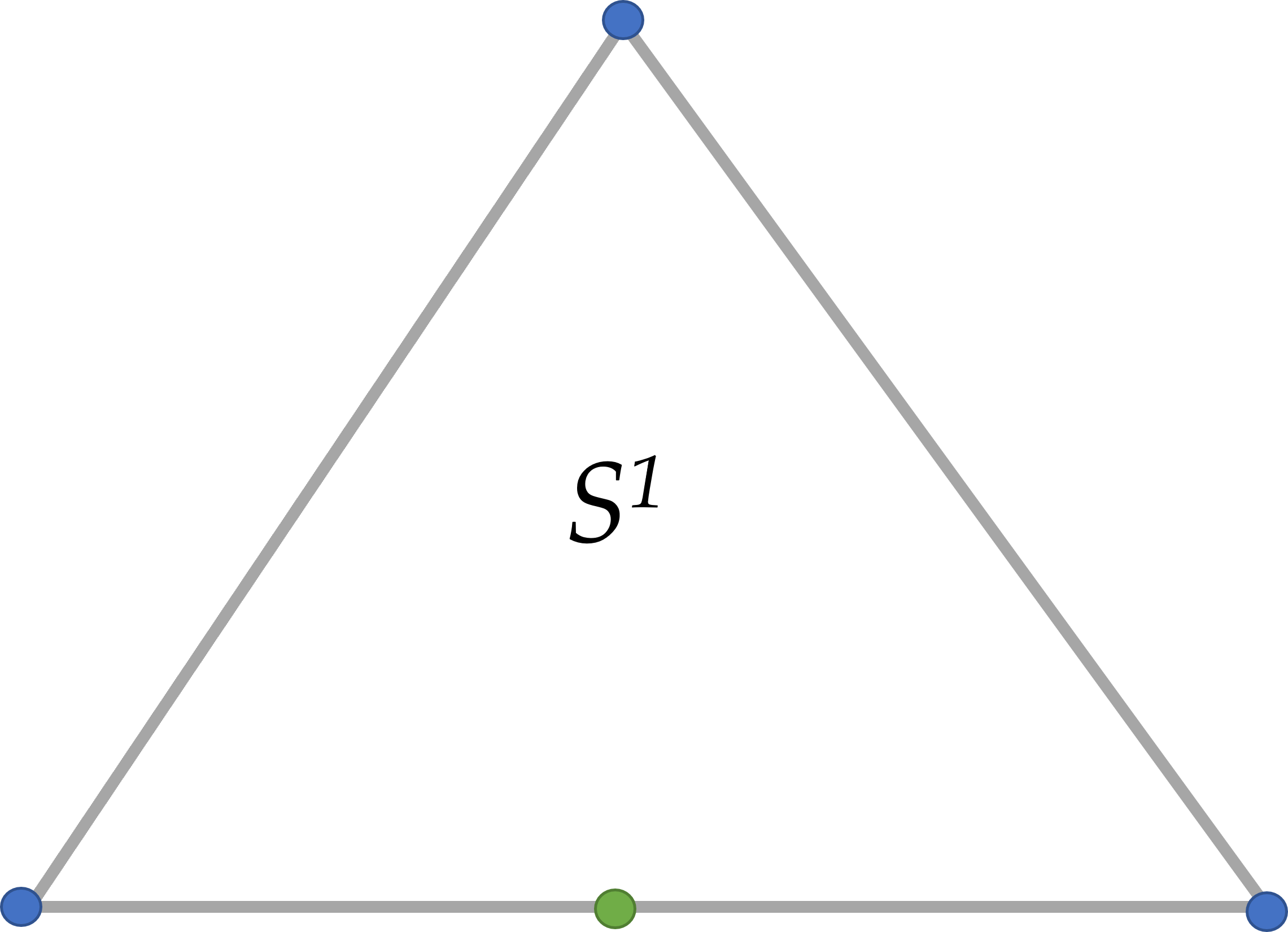}
    \caption{\footnotesize Initial simplex.}
            \label{fig:alg1}
    \end{subfigure}\hfill
    \begin{subfigure}[t]{0.15\textwidth}
         \centering
        \includegraphics[width=0.99\linewidth]{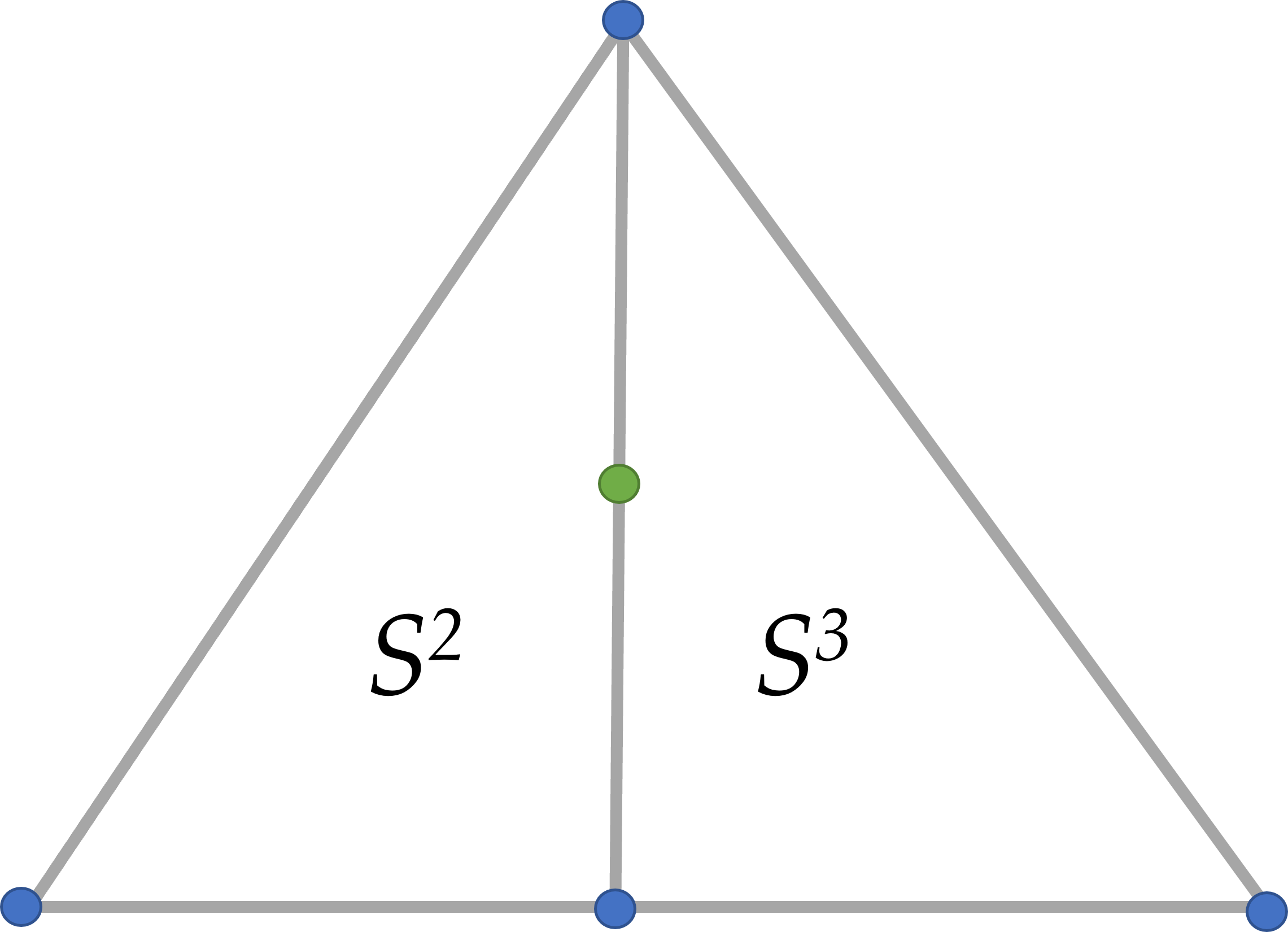}
    \caption{\footnotesize Iteration 1.}
        \label{fig:alg2}
    \end{subfigure}
    \hfill
    \begin{subfigure}[t]{0.15\textwidth}
         \centering
        \includegraphics[width=0.99\linewidth]{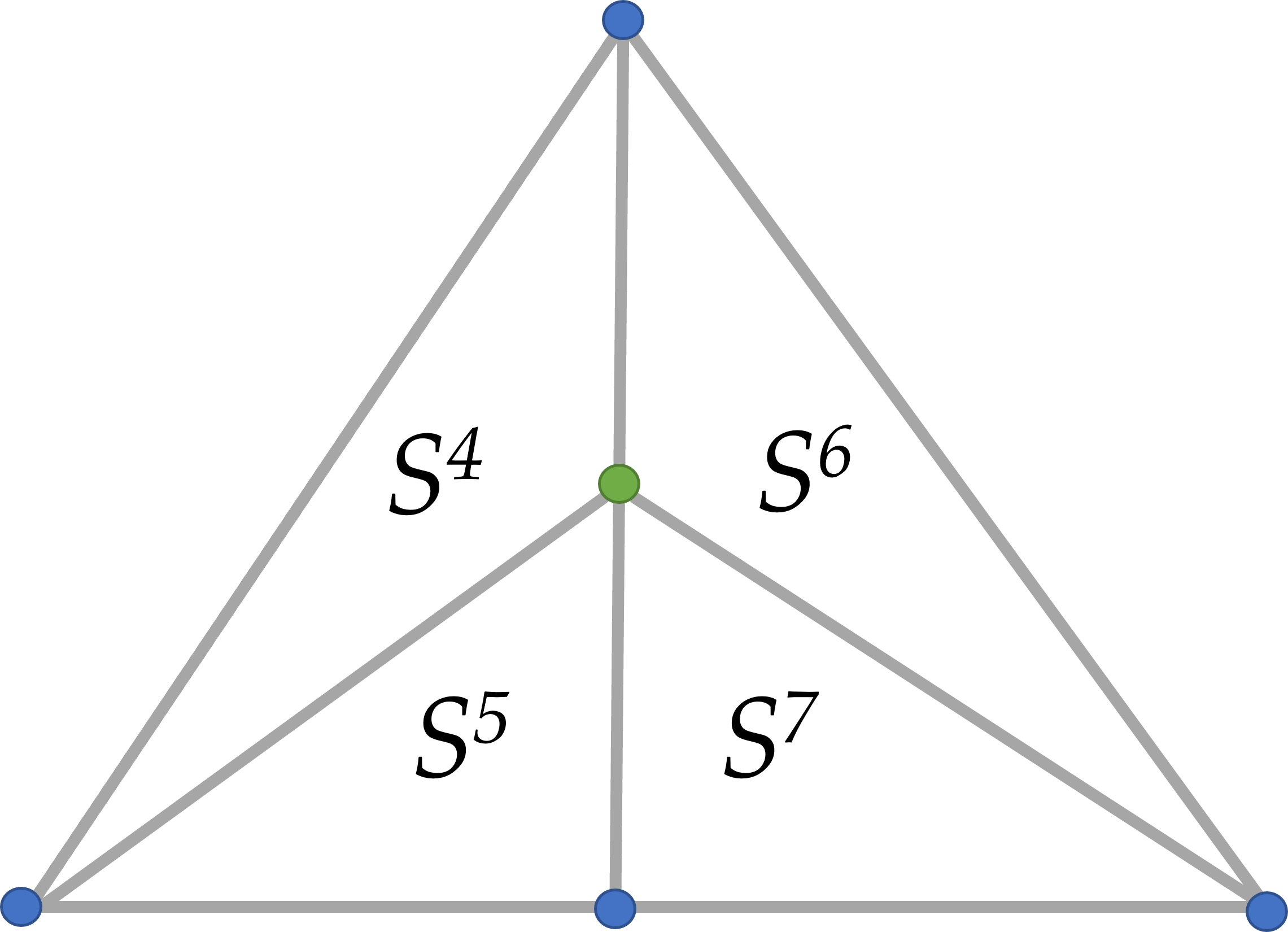}
    \caption{\footnotesize Iteration 2.}
        \label{fig:alg4}
    \end{subfigure}
    \caption{Illustration of multiple splits in Algorithm \ref{alg:update_segments} over two iterations, green indicates the new sample $\w'$.}
    \label{fig:alg_example_splits}
\end{figure}

\paragraph{Example Illustration}
\label{par:example}
To conclude the description of the approach, Figure \ref{fig:alg_example} provides an illustration of the Pareto-approximation constructed by Algorithm \ref{alg:Pareto_Explore} for two cost functions.  
Since we assumed the sum of scalarization weights to be one, we can simplify the notation and represent each weight vector $\w$ simply by its first entry $w$.
Initially, we compute the two basis solutions $w=0$ and $w=1$ to get their cost distributions $\Cvec_{\I}(0)$ and $\Cvec_{\I}(1)$. The algorithm picks $w=.5$ as the midpoint of the only simplex $\{0,1\}$, computes the cost distribution $\Cvec_{\I}(.5)$ and subsequently adds $w=.5$ to $\Omega$. In the second iteration, there are two simplexes $\{0,.5\}$ and $\{.5,1\}$ where the former has a larger dispersion (as illustrated in Figure \ref{fig:alg_example}). Thus, the midpoint $w=.25$ is chosen. 
However, the resulting distribution $\Cvec_{\I}(.25)$ overlaps with $\Cvec_{\I}(0)$, resulting in a high probability of failing a hypothesis test. Thus, $w=.25$ is \textit{not} added to $\Omega$, and future calls of $\mathtt{Find\_New\_Weight}$ will not return the midpoint of the simplex $\{0,.25\}$. Finally, in iteration 3, the midpoint $w=.375$ of simplex $\{.25,.5\}$ is selected. The resulting distribution $\Cvec_{\I}(.375)$ passes the selection criterion and thus is added to the sample set $\Omega$.
\begin{figure}[t]
    \centering
    \begin{subfigure}[t]{0.2\textwidth}
            \centering
             \includegraphics[width=0.95\linewidth]{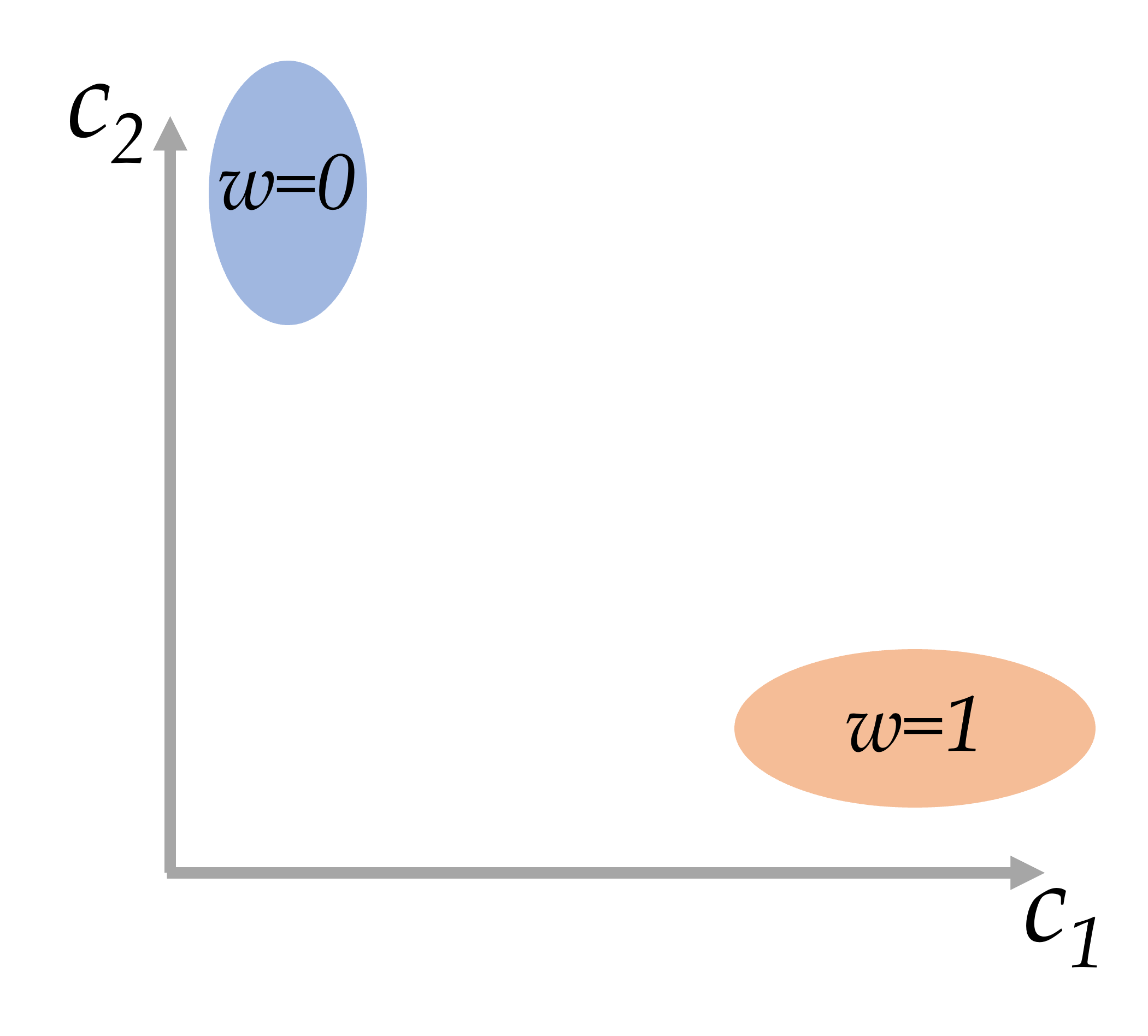}
    \caption{\footnotesize Basic weights $w=0$ and $w=1$.}
            \label{fig:alg_example1}
    \end{subfigure}\qquad
    \begin{subfigure}[t]{0.2\textwidth}
         \centering
        \includegraphics[width=0.95\linewidth]{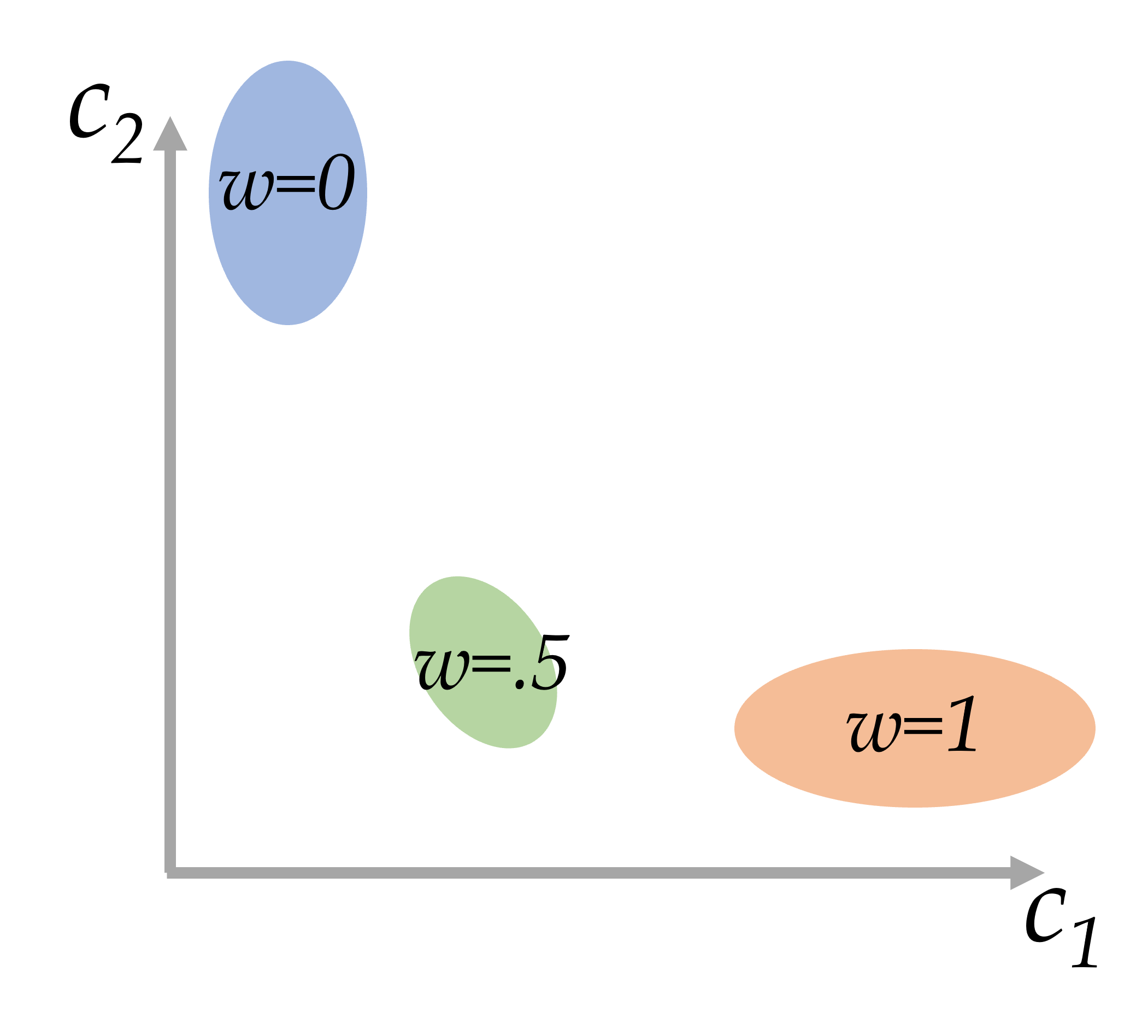}
    \caption{\footnotesize Iteration 1, add $w=.5$.}
        \label{fig:alg_example2}
    \end{subfigure}\\
    \begin{subfigure}[t]{0.2\textwidth}
            \centering
             \includegraphics[width=0.95\linewidth]{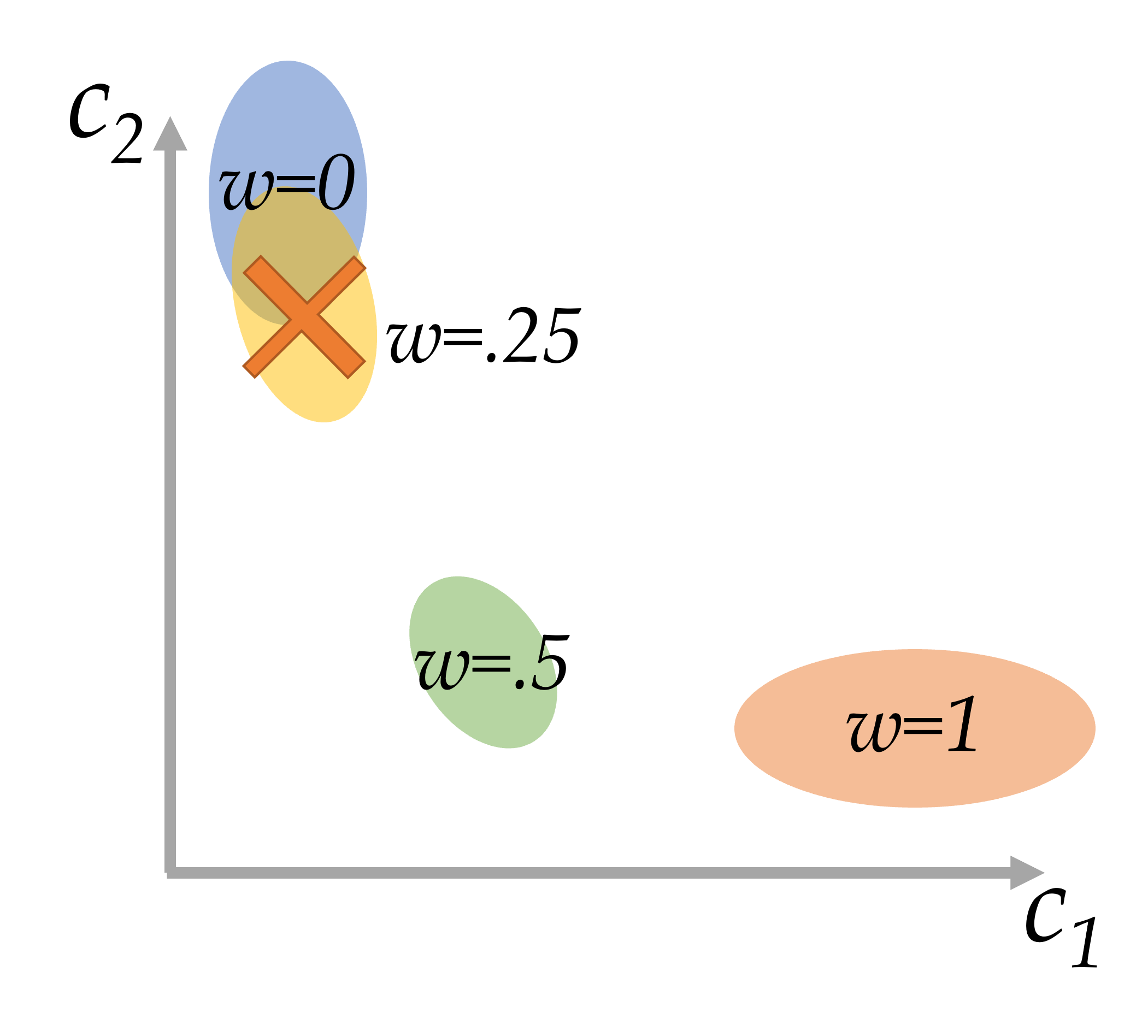}
    \caption{\footnotesize Iteration 2, reject $w=.25$.}
            \label{fig:alg_example3}
    \end{subfigure}\qquad
    \begin{subfigure}[t]{0.2\textwidth}
         \centering
        \includegraphics[width=0.95\linewidth]{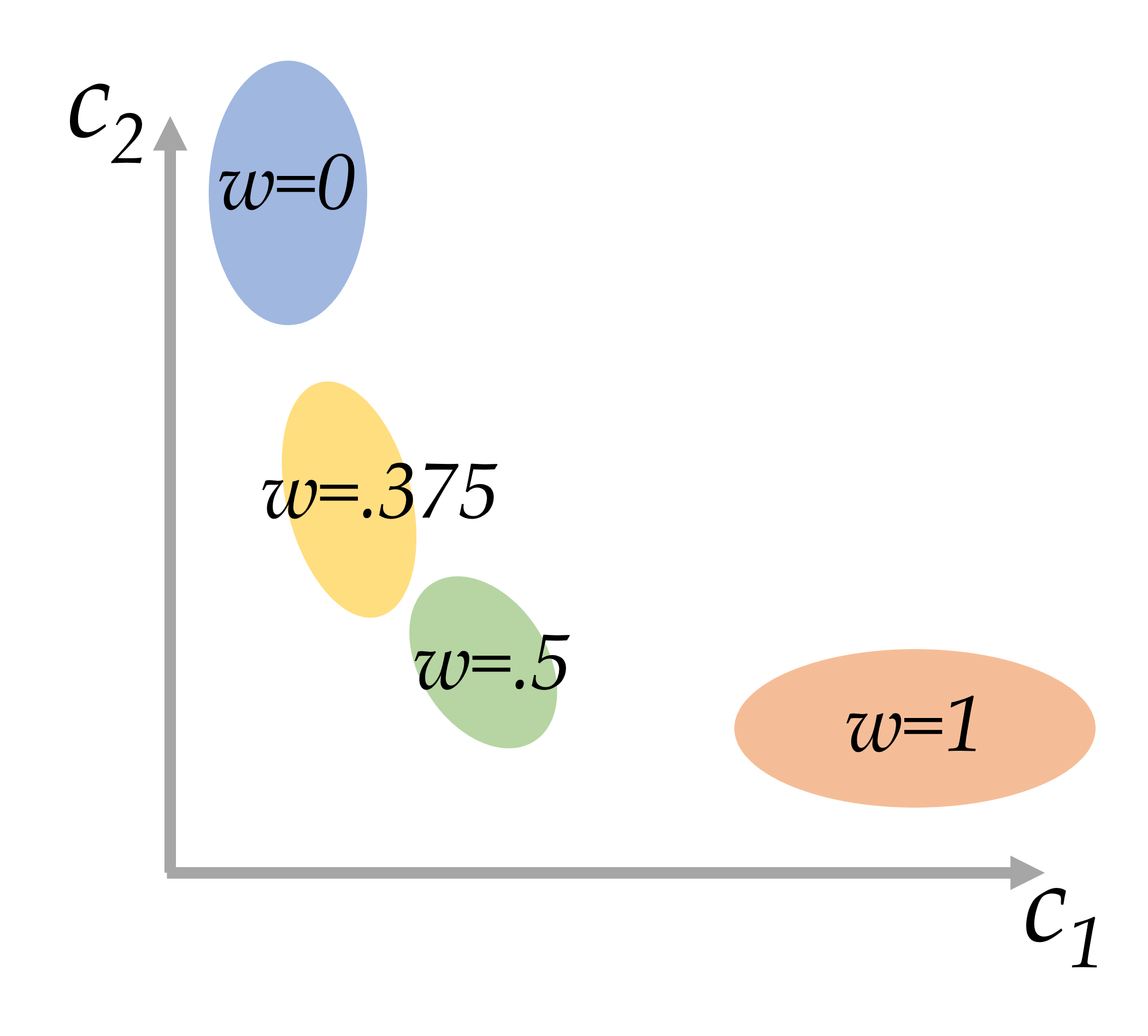}
    \caption{\footnotesize Iteration 3, add $w=.375$.}
        \label{fig:alg_example4}
    \end{subfigure}
    \caption{Illustration of Algorithm \ref{alg:Pareto_Explore}. Colored ellipses show the estimated cost distributions $\Cvec_{\I}(w)$ for iteratively sampled weights. }
    \label{fig:alg_example}
\end{figure}

\subsection{Theoretical Results}
\label{sec:theoretical}
In this section we establish several theoretical properties of of Algorithm \ref{alg:Pareto_Explore}.
We begin with characterizing the runtime. Let $t_{\pi}(m,n)$ be the runtime of the MRPD-solver $\pi$ for a problem with $m$ robots and $n$ tasks. Algorithm \ref{alg:Pareto_Explore} requires exactly $K$ calls of the MRPD solver for each of the $\eta$ training instances, yielding a runtime of $O(K \eta t_{\pi}(m,n))$. Thus, given a polynomial time approximation or heuristic for MRPD, the runtime remains in polynomial time.

Next, we show that the proposed algorithm satisfies all three conditions formulated in Problem \ref{prob:MOMRPD}.
Thus, let $\Pi=\{\pi^1, \dots, \pi^k\}$ be the policies that correspond to the weights $\Omega=\{\w^1, \dots, \w^k\}$ returned by Algorithm \ref{alg:Pareto_Explore}.

\bigskip
\paragraph{Optimal solutions}
The first condition is that the policies yield \textit{approximately} Pareto-optimal system behaviour. This strongly depends on the chosen MRPD policy. In Section \ref{sec:example_costs} we show the implementation of different MRPD cost functions to solve \eqref{eq:MRPD_LSMOOP} using a state-of-the-art algorithm \cite{alonso2017demand}, which finds optimal solutions given sufficient computational budget.

\bigskip
\paragraph{Minimal Dispersion}
Secondly, the solution $\Pi$ should minimize \textit{dispersion}.
To that end, we establish completeness of Algorithm \ref{alg:Pareto_Explore}, \textit{i.e.,} for a sufficiently large budget $K$ the solution found by the algorithm has the minimal attainable dispersion.
We begin with a supporting lemma.

\begin{lemma}[Piece-wise constant]
\label{lem:piececonstant}
The function $\vect{c}(\pi(\w), \T)$ is piece-wise constant over $\w$ for a fixed and finite sequence of tasks $\T$.
\end{lemma}
\begin{proof}
The solution space for a policy $\pi(\w)$ is the set of finite sequences of vertices on the graph $G$ for each robot, which itself is finite in size. Thus, the cost must be piece-wise constant.
\end{proof}

Lemma \ref{lem:piececonstant} implies that, given a fixed MRPD policy $\pi(\w)$, there exists a finite sized set of weights $\Omega^*=\{\w^1, \w^2, \dots\}$ such that every $\vect{c}(\pi(\w), \T)$ is attained by a policy $\pi(\w)$ for exactly one element in $\Omega^*$. Let $\Pi^*$ be the corresponding set of policies, which by definition
achieves the smallest possible dispersion $D(\Pi^*)$.
Despite Lemma \ref{lem:piececonstant} there can exist solutions that are only attained for a singleton $\w$, making any sampling based method only asymptotically complete.
This motivates the following assumption.

\begin{assumption}[Non-singleton solution weights]
\label{as:volume}
Let $\Gamma^*$ be the set of optimal solutions to \eqref{eq:MRPD_LSMOOP}. We assume that each solution $\cvect^i$ in $\Gamma^*$ is optimal for some set of weights $\W^i\subseteq\W$, inscribing a ball in $\R^{n-1}$ of radius $r^i>0$.
\end{assumption}
Thus, each solution can be found by any weight $\w^i$ lying in $\W^i$.
The dimension $n-1$ of the ball comes from the weight space $\W$ being a subset of $\R^n$, constrained by one equality.
We can now establish the second key result for the proposed algorithm.

\begin{theorem}[Completeness] 
\label{thm:complete}
Under Assumption \ref{as:volume}, Algorithm \ref{alg:Pareto_Explore} is complete, \textit{i.e.,} finds all solutions $\Gamma^*$ given a sufficiently large but finite budget $K$.
\end{theorem}
\begin{proof}
To prove the theorem, we first establish two claims: 
\begin{claim}
\label{claim_expand}
$\mathtt{Find\_New\_Weight}$ will expand any edge $(\w^i, \w^j)$ where $\vmu^i\neq \vmu^j$ within finite iterations.
\end{claim}

\begin{claim}
\label{claim_stop}
Any simplex $s$ is not split further (\textit{i.e.,} disregarded) only if all solutions corresponding to a weight $\w$ in $s$ have already been sampled. 
\end{claim}

Let the input $\Delta$ of Algorithm \ref{alg:Pareto_Explore} be $1$ such that $H_{\Delta}(\w^i,\w^j)=1$ always holds and the H-test never prevents a new weight from being added to $\Omega$.
As a shorthand let $D^{ij}$ denote the pair-wise dispersion $D((\w^i,\w^j))$.

\begin{subproof}[Subproof of Claim \ref{claim_expand}]
To show the first claim, consider some simplex $S$ with an edge $(\w^i, \w^j)$. Since $\vmu^i\neq \vmu^j$, we have $D^{ij}\geq \delta>0$. We now prove that this edge will eventually be expanded. 
At some iteration $k$ let $D^{lp}$ be the current maximimum \textit{discounted pair-wise dispersion}, attained for edge $(\w^l, \w^p)$ in simplex $s'$. Thus, $\mathtt{Find\_New\_Weight}$ will return the weight $\w^q$ that is the midpoint of $(\w^l, \w^p)$. 
This leads to two cases: 1) The solution for $\w^q$ is a new solution, 2) the solution is identical to either the solution corresponding to $\w^l$ or $\w^p$.
In the first case, the algorithm explores a new element of the expected Pareto-front. Since the set of all solutions is finite, this can only happen for a finite number of iterations.
However, no immediate progress is made in the second case when $\vmu^q$ is not a new solution.
Without loss of generality, we order $\w^l$ and $\w^p$ such that $\vmu^q=\vmu^p$ holds.
We now show that edges of the new simplexes created by 
Algorithm \ref{alg:update_segments} have a smaller discounted pair-wise dispersion than the old edge $(\w^l,\w^p)$.
The new edges are
$(\w^l,\w^q)$, $(\w^q,\w^p)$, and $(\w^q,\w^r)$ for any $\w^r$ in $s'$ where $\w^r\neq \w^l$ and $\w^r\neq \w^p$.
First, $D^{qp}=0$ trivially holds since $\vmu^q=\vmu^p$. Further, the discount factor for the new edge $(\w^l,\w^q)$ equals $\alpha(\w^l,\w^p) + 1$. Hence, $D^{lq} = \nicefrac{1}{2}D^{lp}$. Lastly, there are $n-4$ other edges $(\w^q,\w^r)$. However, since $\vmu^q=\vmu^p$, the discounted pair-wise dispersion $D^{qr}$ is equal to the existing edge $(\w^p,\w^r)$. Finally, any future iteration returning the midpoint of either $(\w^p,\w^r)$ or $(\w^q,\w^r)$ will increment the discount factor of both edges, and thus update their discounted pair-wise dispersion.

Thus, each iteration removes a current maximizer of the discounted pair-wise dispersion, and introduces a) two new edges with $D^{qp}=0$ and $D^{lq} = \nicefrac{1}{2}D^{lp}$, b) $n-4$ new edges that are \textit{coupled} to existing edges, \textit{i.e.,} share the same discount factor $\alpha$. Hence, letting $N$ be the total number of edges among all simplexes, it takes at most 
$(N-1)\log_2 \nicefrac{D^{lp}}{\delta}$ iterations until the edge $(\w^i, \w^j)$ becomes the maximizer and its midpoint is returned by $\mathtt{Find\_New\_Weight}$.
\end{subproof}

\begin{subproof}[Subproof of Claim \ref{claim_stop}]
The second claim ensures that no solution is missed by not expanding a simplex further. Following the first claim, a simplex $s$ is not expanded only if $D^{ij}=0$ for all edges $(\w^i, \w^j)$ of $s$. 
Since we set $\Delta=1$ it must always hold that $H_{\Delta}(\w^i,\w^j)=1$. Hence, we have $D^{ij}=0$ if and only if $\vmu^i= \vmu^j$, following the definition in equation \eqref{eq:disc_disp}.
If this holds for all edges in $s$, all vertices of the simplex must have equal solutions $\Cvec_{\mathcal{I}}$. Since the set of weights $\W^i$ yielding the same solution is convex \cite{wilde2019bayesian}, all weights in the interior of $s$ must also correspond to the same solution $\Cvec_{\mathcal{I}}$. Thus, expanding $s$ further cannot yield any solution that has not been sampled yet.
\end{subproof}

In conclusion, Claim \ref{claim_expand} ensures that every edge with non-equal solutions is expanded within a finite number of iterations and Claim \ref{claim_stop} guarantees we only disregard parts of the weight space when they cannot lead to finding a new solution.
Lastly, we need to establish that it is sufficient to only add new samples on the midpoint of edges. We notice that by adding new weights as the midpoint of some edge in simplex $s$, the newly created simplexes $s^i$ and $s^j$ have edges that pass through the circumcenter of $s$. By Assumption \ref{as:volume}, there will eventually be an edge $(\w^i,\w^j)$ passing through each set $\W^i$ corresponding to a solution. Thus, by searching over midpoints of edges, a sample will be placed in each set $\W^i$ after finite iterations, concluding the proof.
\end{proof}

To summarize, Theorem \ref{thm:complete} ensures that Algorithm \ref{alg:Pareto_Explore} achieves minimum dispersion for a sufficiently large but finite $K$.  

\begin{remark*}[Necessity of discount factor]
An intuitive simplification of the proposed method could be to only consider the edge with maximum pair-wise dispersion without a discount factor. However, such an approach is not complete, as we will illustrate in a simple counterexample.
Consider an instance of Problem \ref{prob:MOMRPD} with only one task and a single robot located at the task's pickup location. For three given cost functions let there be only four unique solutions on how the robot can move from the pickup location to the goal, attaining the following mean cost vectors: $\vmu^1 = [0\ 5\ 2]$, $\vmu^2 = [5\ 0\ 2]$, $\vmu^3 = [10\ 10\ 1]$, and $\vmu^4 = [6\ 6\ 1.1]$ (taking the mean here is trivial since the task arrival is fixed to be deterministic). We notice that all four vectors are Pareto-optimal since they are not dominated by any other vector.
Algorithm \ref{alg:Pareto_Explore} begins with the basis weights $\vect{w}^1=[1\ 0\ 0]$, $\vect{w}^2=[0\ 1\ 0]$, $\vect{w}^3=[0\ 0\ 1]$ and computes the corresponding solutions, which in this case will be $\vmu^1$,$\vmu^2$ and $\vmu^3$, respectively.
Among these three solutions, the maximum pairwise dispersion is found between $\w^1$ and $\w^2$.
However, observe that the optimal solution for any convex combination $w'=[\lambda\ 1-\lambda \ 0]$ of $\w^1$ and $\w^2$ is either $\vmu^1$ or $\vmu^2$. 
Hence, placing a sample on the edge between these two weights does not discover a new solution. Thus, without the discount factor, the algorithm will perpetually place new weights on the line between $\w^1$ and $\w^2$ without ever finding the fourth solution $\vmu^4$ and therefore does not converge.
\end{remark*}

\paragraph{Statistically different solutions}
Lastly, we consider the third property that cost distributions of different policies are statistically distinguishable. 
\begin{lemma}[Distinguishable solutions]
    The policies $\Pi=\{\pi^1, \dots, \pi^k\}$ have statistically significantly different cost distributions $\cvect(\pi,\T)$.
\end{lemma}
\begin{proof}
    Algorithm \ref{alg:update_segments}, ensure that only weights $\w'$ are added to the solution set when the sampled costs $\Cvec_{\I}(\w')$ passes an H-test against all already sampled solutions $\w$ with threshold $\Delta$. For a sufficiently large $\eta$, the sampled cost vectors $\Cvec_{\I}(\w)$ approximate the distributions $\cvect(\pi(\w),\T)$ for all $\w$. Thus, the choice of $\Delta$ provides an upper bound on how likely it is that two solutions are not statistically significantly different.
\end{proof}

In conclusion, we have shown that Algorithm \ref{alg:Pareto_Explore} satisfies all three requirements posed in Problem \ref{prob:MOMRPD}. Next, we will show example configurations for different MRPD objectives.

\section{Exemplary MO-MRPD configuration}
\label{sec:example_costs}
In this section, we show an approach to solving the scalarized MO-MRPD, \textit{i.e.,} the weighted cost function in \eqref{eq:MRPD_LSMOOP} for three objectives.

\subsection{Cost functions}
We consider three MRPD objectives: i) the quality of service (QoS), capturing how timely deliveries are, ii) a social cost for traversing human-centered spaces similar to \cite{wilde2020improving}, and iii) the total travel distance, representing the overall robot traffic and energy consumption.

The quality of service $c^Q(\pi, \T)$ measures the time between each task $T\in\T$ being announced and being completed - also called the \textit{system time} of all tasks. Let $\tau$ be the tour a robot takes following policy $\pi$. We defined $t^f(T, \tau(\pi))$ as the time tour $\tau$ visits the destination vertex of task $T$, and $t^r(T)$ as the release time of the task. 
Given a task's deadline $t^d(T)$, and some large constant $M$, the QoS of a task is then given by
\be
q(T, \tau(\pi))
=
\begin{cases}
    t^f(T, \tau(\pi)) - t^r(T) \text{ if }t^f(T, \tau(\pi))\leq t^d(T),\\
    M\text{ otherwise.}
\end{cases}
\ee

The QoS for all tasks then is
\be
\label{eq:cost_QoS}
c^Q(\pi, \T) = \sum_{T\in \T}q(T, \tau(\pi)).
\ee
The second cost is the social cost, capturing the intrusiveness of a robot tour into human-centered parts of the environment. To this end a subset $E'$ of the edges on the graph are labelled as \textit{'avoid robot traffic'}, which we encapsulate in a binary indicator function $\phi(e)$.
The social cost $c^S(\tau)$ then counts how many such labelled edges are visited by a tour, \ie appear in the tours edge sequence $E(\tau)$:
\be
\label{eq:cost_route}
c^S(\tau) = \sum_{e\in E(\tau)} \phi(e).
\ee
Lastly, the total distance $c^T(\tau)$ is simply the length of a robot's tour $\tau$, defined as the sum of all edge lengths:
\be
\label{eq:cost_length}
c^T(\tau) = \sum_{e\in E(\tau)} d(e).
\ee

\subsection{Solving linearly scalarized MO-MRPD for fixed weights}
We now specify a policy for solving \eqref{eq:MRPD_LSMOOP} given fixed weights $\w$. We begin by showing how we compute tours for a single robot and its assigned tasks, before describing the task assignment procedure.
\paragraph{Routing Problem}
Given a set of tasks $\T$ that are assigned to a robot, we need to find a tour starting at the robot's current location $v$ that visits all pickup and dropoff locations $s_j$ and $g_j$ for all $T_j\in \T$, and minimizes 
\be
\label{eq:example_cost}
w_1c^Q(\T,\tau) + w_2c^S(\tau)+w_3c^T(\tau),
\ee
subject to ordering and capacity constraints.

The complexity of ordering and capacity constraints prevents us from using a TSP-approximation algorithm. 
Instead, we construct a new graph $G^{\w}=(V, E, d')$ with the same vertices and edges as the given graph $G$. However, edge costs are defined as $d'(e)=w_1d(e) + w_2\phi(e) + w_3d(e)$ \textit{i.e.,} capture the different costs.
{We notice that when only travelling from some start location to a drop-off location, optimizing QoS and total time is equivalent. However, a minimum-cost tour on $G^{\w}$ does not necessarily minimize \eqref{eq:example_cost}.}  
We compute tours using a two-step heuristic:
First, we find an initial tour using a min-cost insertion approach, with respect to the cost in \eqref{eq:example_cost}. Paths connecting the robot's location and the pickup and drop-off locations are then shortest paths on graph $G^{\w}$, {yet the ordering of locations is picked such that \eqref{eq:example_cost} is minimized to a local optimum.}. Afterwards, we improve the tour using a large neighbourhood search (LNS) \cite{smith2017glns} with random deletion and insertion.
\paragraph{Assignment Problem}

Let $Q$ be a set of newly arrived tasks. The assignment problem decides which robot services which task in $Q$.
In general, our framework is agnostic to the assignment algorithm. A popular state-of-the-art method is the group assignment algorithm from \cite{alonso2017demand}, which we employ in the simulation experiments.
Unlike greedy methods, this framework actively combines different tasks. After collecting newly arrived new tasks for a fixed time period, the algorithm forms groups of tasks that could be serviced  by a single robot while satisfying all deadlines are grouped together. Then, a disjoint subset of these groups is assigned to the robots using an Integer Linear Program (ILP).

\paragraph{Optimality considerations}
{
The presented routing and assignment approach is a heuristic that finds locally optimal solutions. Yet, the approach can easily be modified to be optimal for each individual time instance by i) computing tours using exhaustive search over all possible orderings of pickup and drop-off locations, and ii) running the group assignment with groups sizes up to the number of currently outstanding tasks. Then, we would compute Pareto-optimal solutions at the current time of planning for newly arrived tasks.
However, the practical benefits are limited due to the poor scalability of exhaustive search. Moreover, the heuristic implementation of the group assignment framework was shown to be effective in finding high quality solutions for large scale problems \cite{alonso2017demand}.
}

\section{Evaluation}
We evaluate our proposed method for finding a set of MO-MRPD policies in a set of simulation experiments.
We consider instances of MO-MRPD with two and three objectives, namely the QoS, social cost and total length cost described in Section \ref{sec:example_costs}.

\subsection{Experiment Setup}
\paragraph{Environment and MRPD settings}
We consider two different environments: a real-world office floorplan as well as an artificial map with a central lobby where robot traffic is undesired, both are shown in Figure \ref{fig:env_examples}.
In both maps, the task sequences are generated using a Poisson process, while pickup and dropoff locations are sampled uniformly random from a set of predetermined locations. Task deadlines are manually chosen such that all tasks can be serviced on time when only optimizing for QoS, but deadlines might be missed when considering the other two costs.
Throughout the experiment, we vary the MRPD system to feature between $2$ and $8$ robots with each a capacity of $\kappa=4$ for servicing $100$ to $200$ tasks.
\begin{figure}[t]
    \centering
    \begin{subfigure}[t]{0.49\textwidth}
            \centering
             \includegraphics[width=0.95\linewidth]{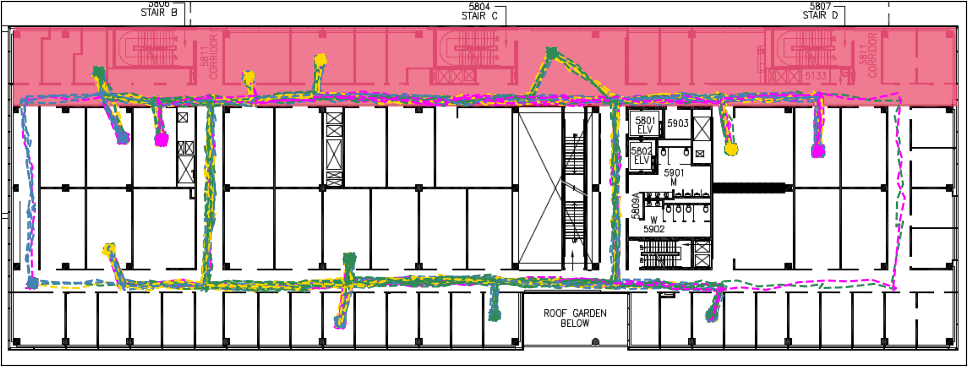}
    \caption{Office environment.}
    \end{subfigure}\qquad
    \begin{subfigure}[t]{0.3\textwidth}
         \centering
        \includegraphics[width=0.75\linewidth]{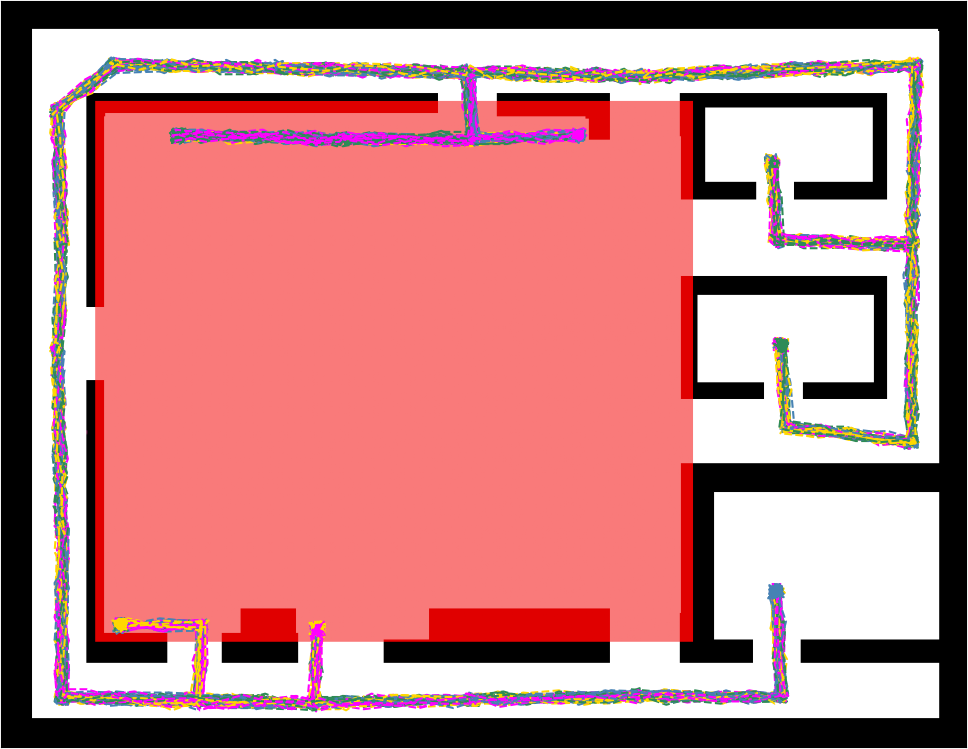}
    \caption{Lobby environment.}
    \end{subfigure}

    \caption{Simulation environments for the experiments. Free space is shown in white, static obstacles in black. The red area indicates parts of the environment where robot traffic is undesired (\textit{the upper floor and the main lobby, respectively}). {The dashed lines show an robot routes for an example MO-MRPD solution: in (a) QoS is of high priority, such that there is a lot of robot traffic in the social space while in (b) the robots avoid the lobby as much as possible.}.}
    \label{fig:env_examples}
\end{figure}

\paragraph{Algorithm settings}
Algorithm \ref{alg:Pareto_Explore} uses varying computation budgets $K$, the threshold for the h-test is set to $\Delta=.1$, and the number of instances is $\eta=20$.

\paragraph{Baseline Comparison}
We compare the proposed adaptive sampling method (Adaptive Sampling - $\AS$) against several baselines. 
We consider two variants of finding policies by placing scalarization weights \textit{uniformly}.
The first places $K$ weights uniformly ($\Uni$), where $K$ is the budget given to $\AS$. Since our Algorithm has an early-stopping mechanism such that no samples are added when they fail the h-test it might use fewer than $K$ samples. 
Thus, we additionally compare with uniform sampling ($\Unib$) placing as many samples as were returned by $\AS$.

The third baseline, proposed in \cite{bhonsale2022ellipsoid,mores2023multi}, is closely related to our work. Using a divide-and-conquer approach it divides the set of weights into regular simplexes and places samples on its vertices. We denote this algorithm as $\baseline$. Similar to our approach, they stop dividing a simplex further when the associated solutions are statistically too similar, evaluated by the overlap of confidence ellipsoids. We slightly modify their approach to use our H-test as the stopping criterion for a better comparison. A key difference is that their method does not use a greedy objective to select the next simplex to explore and keeps exploring until every branch of computation reaches the stopping criterion. We limit the number of computations for $\baseline$ to the same budget $K$ given to $\AS$, and let it explore branches in a breath-first-search manner. 

\paragraph{Separation of training and testing instances}
The proposed algorithm as well as the $\baseline$ baseline use several problem instances to estimate the statistical variance of the objectives. In Algorithm \ref{alg:Pareto_Explore} this is denoted by the parameter $\eta$.
For the main simulation results we use $\eta=20$ random instances to run the algorithms. For evaluation we use $\eta^{\mathtt{test}}=20$ different randomly generated instances.
At the end of the evaluation section, we show that our method and evaluation are robust towards changes in $\eta$ and $\eta^{\mathtt{test}}$.

\paragraph{Evaluation measures}
We now introduce our performance measure to evaluate how well a set of sampled policies $\Omega=\{\w^1,\dots,\w^k\}$ with corresponding cost distributions $\Cvec_{\I}(\w^1),\dots,\Cvec_{\I}(\w^k)$ solves Problem \ref{prob:MOMRPD}. 
First, we normalize the cost vectors. Given the means $\vmu^1,\dots, \vmu^k$, let ${\mu}_i^{\min}$ and
${\mu}_i^{\max}$ be the minimum and maximum mean values for cost function $i$.
We then normalize each cost $c_i^j$ for all $j=1,\dots, k$ using the minimum and maximum means, \textit{i.e.,}
\be
\label{eq:normalize}
\bar c^j_i = \frac{( c^j_i-{\mu}_i^{\min})}{( {\mu}_i^{\max}-{\mu}_i^{\min})}.
\ee
Given these normalized cost, we denote the their distributions as $\bar\Cvec_{\I}(\w^1),\dots,\bar\Cvec_{\I}(\w^k)$ with means $\bar \vmu^1,\dots, \bar\vmu^k$.
Algorithm performance is then captured by four measures:
\begin{enumerate}
    \item $\mathtt{Hypothesis\_Error}$. The hypothesis error characterizes how \emph{statistically distinguishable} the behaviour produced by the different policies is. Thus, we let $h(\w^i, \w^j)$ be the probability of failing a type-1 hypothesis error between distributions $\bar\Cvec_{\I}(\w^i)$ and $\bar\Cvec_{\I}(\w^j)$.

    \item $\mathtt{Dispersion}$. The dispersion captures how well the expected Pareto-front is sampled, \textit{i.e.,} the size of gaps between sampled points on the Pareto-front. Implementing Definition \ref{def:dispersion} is impractical since the set of Pareto-optimal solutions is not available. Thus, we consider the following approximation: Given a set of policies and their mean cost vectors $\bar\vmu^1,\dots, \bar\vmu^k$, the \textit{approximated dispersion} is the radius of the largest ball such that a) the center $\vect{p}$ of the ball is located on a line connecting two points $\bar\vmu^i$ and $ \bar\vmu^j$, b) $\vect{p}$ is not \textit{dominated} by any other point $\bar\vmu^q$, and c) the ball does not contain any mean cost vector $\bar\vmu^q$ for all $i,j,q=1,\dots,k$.

    \item $\mathtt{Variance}$. Variance captures how homogenous the mean cost vectors $\bar\vmu^1,\dots, \bar\vmu^k$ are placed, \textit{i.e.,} how \textit{uniformly} we cover the expected Pareto-front.
    We approximate this measure by computing a minimum spanning tree (MST) where $\bar\vmu^1,\dots, \bar\vmu^k$ correspond to vertices and edge lengths to the euclidean pair-wise distances. The variance is then the variance of the edge lengths in the MST.
    
    \item $\mathtt{Coverage}$. Coverage \cite{zitzler1999multiobjective} captures how dominant the computed solutions are. Thus, we computes the volume of the subset $[0,1]^n$ that is \emph{not} dominated by the vectors $\bar\vmu^1,\dots,\bar\vmu^k$. We approximate the measure by sampling points in $[0,1]^n$ and checking if they are dominated by any $\bar\vmu^i$ for $i=1,\dots,k$.
  
\end{enumerate}
For all measures, a smaller value indicates better performance.
To ensure meaningful comparison between different experiment setups, we need to normalize the coverage measure. The best achievable coverage can vary greatly between problem setups since the expected Pareto-front may take different shapes, despite the normalized cost vectors.
Hence, we use the variance achieved by $\Uni$ as a normalizing constant for each experiment.
The other measures do not require normalization: The hypothesis error is an absolute measure. Dispersion and variance are comparable given the normalization of cost vectors in equation \eqref{eq:normalize}.

\subsection{Qualitative Analysis}

\begin{figure*}[t]
    \centering
    \begin{subfigure}[t]{0.32\textwidth}
            \centering
             \includegraphics[width=0.99\linewidth]{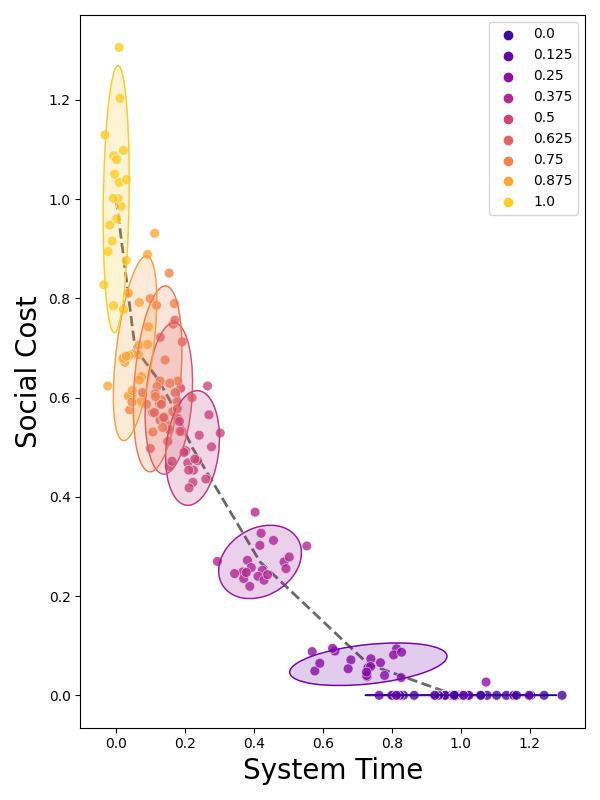}
    \caption{Uniform sampling ($\Unib$)}
            \label{fig:office_uniform}
    \end{subfigure}\hfill
    \begin{subfigure}[t]{0.32\textwidth}
         \centering
        \includegraphics[width=0.99\linewidth]{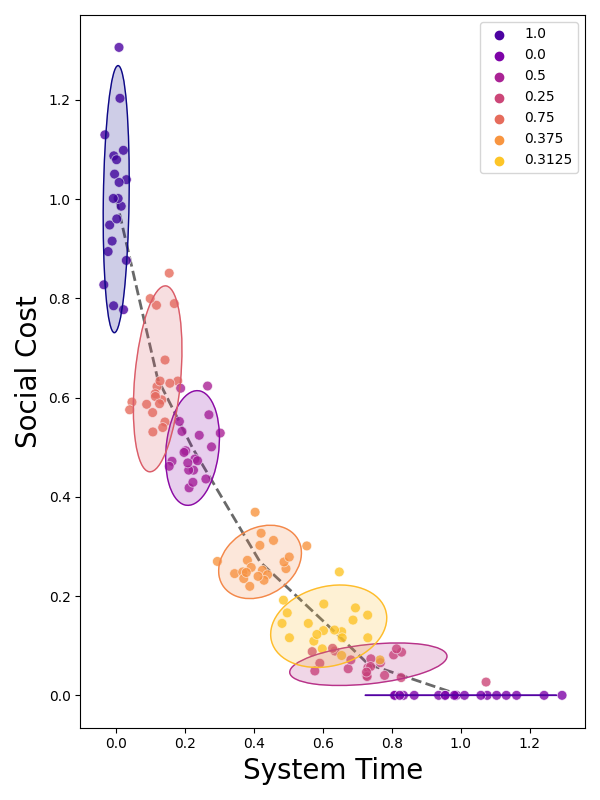}
    \caption{Baseline ($\baseline$)}
        \label{fig:office_baseline}
    \end{subfigure}
    \hfill
    \begin{subfigure}[t]{0.32\textwidth}
         \centering
        \includegraphics[width=0.99\linewidth]{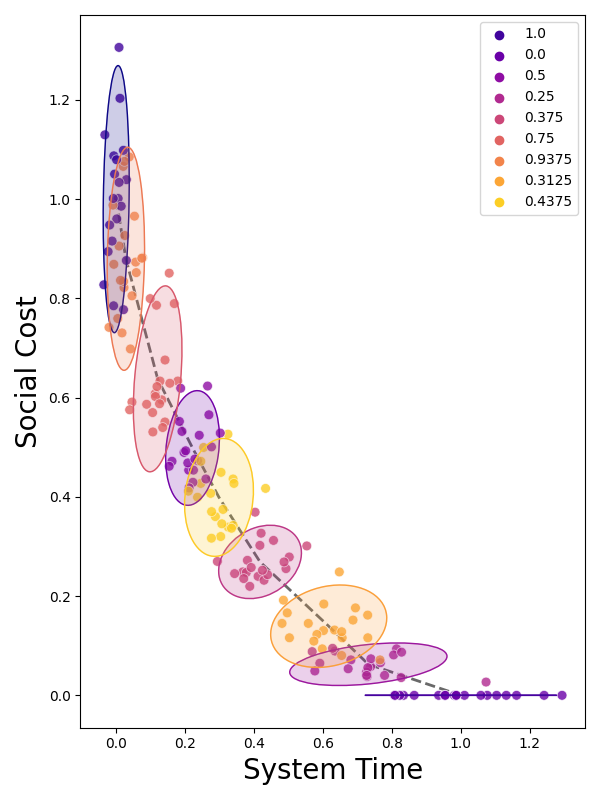}
    \caption{Proposed Adaptive Sampling ($\AS$).}
        \label{fig:office_PAAB}
    \end{subfigure}
    \caption{Example of Pareto approximations with budget $k=10$ for $4$ robots, $100$ tasks, operating in the lobby environment for a time horizon of $t=3000$.
    Each point cloud shows a cost vector $\cvect(\w, \T)$ for a policy $\pi(\w)$, ellipses illustrate 2 standard deviations of the sampling distribution $\Cvec_{\mathtt{I}}(\w)$. {The grey line interpolating cost means shows an approximation of the expected Pareto front}. The ordering in the legends corresponds to the order in which samples were placed.}
    \label{fig:Eval_Example}
\end{figure*}

\begin{figure*}[t]
    \centering
    \begin{subfigure}[t]{0.32\textwidth}
         \centering
        \includegraphics[width=0.99\linewidth]{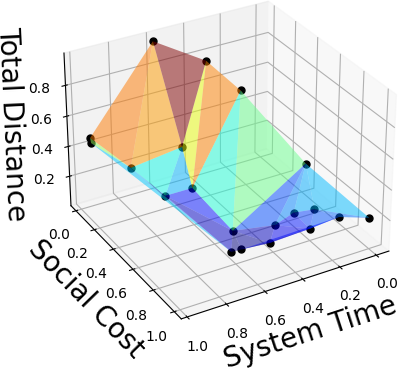}
    \caption{Uniform sampling ($\Unib$)}
        \label{fig:lobby3D_uni}
    \end{subfigure}
    \hfill
    \begin{subfigure}[t]{0.32\textwidth}
         \centering
        \includegraphics[width=0.99\linewidth]{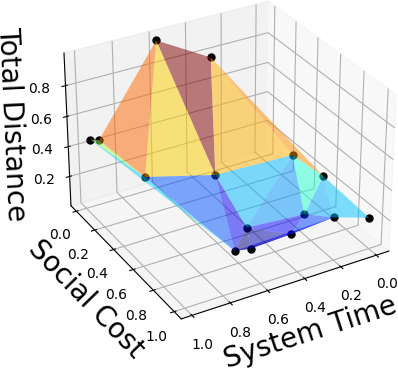}
    \caption{Baseline ($\baseline$).}
        \label{fig:lobby3D_uni2}
    \end{subfigure}
    \hfill
    \begin{subfigure}[t]{0.32\textwidth}
         \centering
        \includegraphics[width=0.99\linewidth]{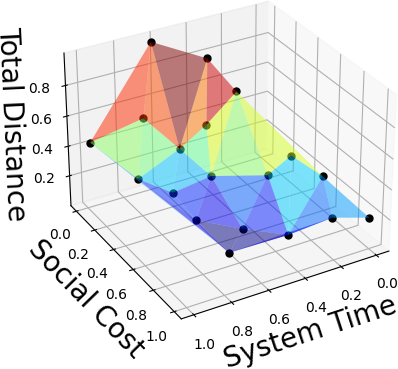}
    \caption{Proposed Adaptive Sampling ($\AS$).}
        \label{fig:lobby3D_proposed}
    \end{subfigure}
    \caption{Example of Pareto approximations with budget $k=30$ for $8$ robots, $200$ tasks, operating in the office environment for a time horizon of $t=3000$. Points show the mean cost vectors $\vmu(\w)$ and the 3D surface the Delaunay triangulation between vectors. Surface colors only support the 3D illustration.}
    \label{fig:example_results3D}
\end{figure*}

\paragraph{Two Objectives}
First, we consider an example experiment in the office environment with two objectives (QoS and social cost). Figure \ref{fig:Eval_Example} shows an exemplary comparison of the policies generated by the different approaches for a budget of $K=10$. We plot the distributions $\bar\Cvec(\w)$ for each sampled weight $\w\in\Omega$, together with ellipses showing two standard deviations around the distribution means $\bar\vmu$. The black line shows the linear interpolation between the means.

We observe that the proposed method places samples most homogeneously covering large parts of the expected Pareto front (dispersion $.25$) and with only marginal overlap between the distributions (mean h-test $.05$). In contrast, $\Unib$, exhibits a several gaps between the distributions (dispersion $.37$), while other distributions overlap substantially (mean h-test $.28$). The baseline $\baseline$ achieves a small overlap between solutions (mean h-test $.01$), yet shows the largest gaps between samples (dispersion $.38$). Moreover, despite having the same budget as $\AS$, $\baseline$ places fewer samples (7 compared to 9), indicating a lower efficiency in finding statistically different solutions.
{
We recall that  we can only generate system plans for sampled weights $\w$. Thus, while the interpolation between samples (grey line) is similar between $\baseline$ and $\AS$, only $\AS$ results in more nuanced system plans.}

The plot also shows the order samples are placed using $\AS$, indicated by the color gradient. Starting with the basis solutions that yield the endpoints of the Pareto-fronts, $\AS$ adds new samples placed in the largest current gaps, greedily reducing dispersion.

\paragraph{Three Objectives}
In a second example we show results for the lobby environment considering all three objectives (QoS, social cost and total distance), illustrated in Figure \ref{fig:example_results3D}.

Overall, the result is similar to the 2D case: the proposed method exhibits smaller gaps between the solutions (dispersion $.24$) while avoiding substantial overlap (mean h-test $<.01$). In contrast, $\Unib$ and $\baseline$ oversample solutions with low QoS and low total distance but high transit count, leading to high overlap (mean h-test $.14$ and $.08$, respectively), and larger gaps (dispersion $.32$ and $.37$, respectively). In contrast, $\AS$ produces more evenly spaced solutions over the entire expected Pareto-front, yielding a better approximation.

\subsection{Quantitative Analysis}
Next, we will provide a more in-depth analysis with several quantitative measures for various MRPD problem settings.
We compare our method with all three baselines under different MRPD settings with varying fleet size and task load for both environments.


\begin{figure}[t]
    \centering   
        \includegraphics[width=0.9\linewidth]{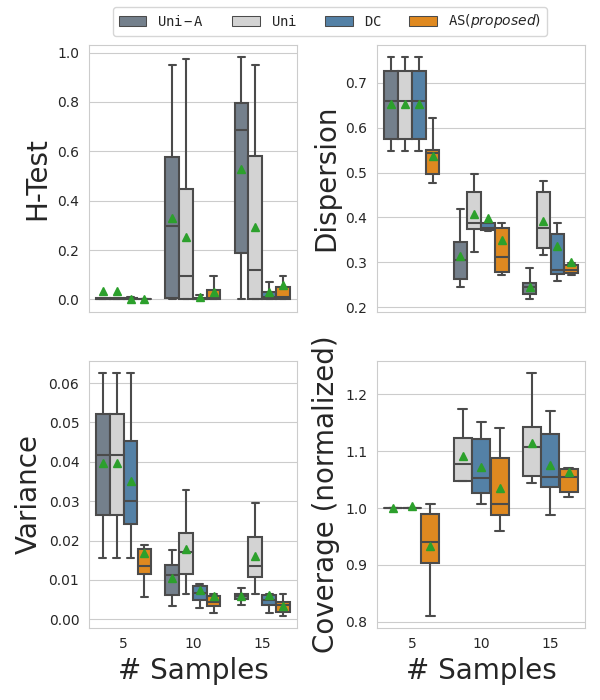}    
    \caption{Results for MRPD with two objectives. Coverage is normalized with respect to $\Uni$.}
    \label{fig:stats_2d}
\end{figure}
\paragraph{Results for two objectives}
We conduct further experiments considering two objectives, QoS and transit count. We employ fleets of $2$, $4$ and $8$ robots to service $100$ tasks in the lobby, and $200$ tasks in the office environments, yielding 6 different problem settings.
The sampling budget $K$ takes values 5, 10 and 15.

We illustrate the quantitative measures in Figure \ref{fig:stats_2d}.
We observe large differences between methods for the $\mathtt{Hypothesis\_Error}$: while all approaches achieve a small error for $K=5$, values increase drastically for both uniform approaches. In contrast, $\baseline$ and the proposed method $\AS$ keep the probability of failing an h-test under the threshold $\Delta=.1$.

For the $\mathtt{Dispersion}$ measure $\AS$ achieves the lowest values for $K=5$ and is second best after $\Uni$ for larger $K$. However, we recall that $\Uni$ always uses the full budget $K$ while $\AS$ often uses fewer samples. Indeed, $\AS$ shows clear advantages over $\Unib$ which uses the same number of samples. Further, $\AS$ also outperforms $\baseline$ with respect to dispersion, yet with a smaller margin than compared to $\Unib$.

The $\mathtt{Variance}$ shows additional insights into how the sampled solutions are spaced along the Pareto-front. Here, the proposed method shows the best performance among all approaches, \textit{i.e.,} it places samples most evenly, yet the margin to $\baseline$ is relatively small.

The $\mathtt{Coverage}$ omits the result for $\Uni$ since we normalize to its values. We observe that $\AS$ outperforms $\Unib$ and $\baseline$ for all budgets; however, the difference decreases for larger $K$.
We observe that the normalized coverage values increase with larger $K$ for $\Unib$, $\baseline$ and $\AS$. While coverage itself decreases monotonically with larger $K$, the normalized value, \textit{i.e.,} the relative value compared to $\Uni$, becomes poorer.
Lastly, we notice that for $K=5$ the coverage of $\AS$ lies below $1.0$, indicating a better value than obtained by $\Uni$

In summary, the proposed method outperforms $\Uni$ and $\baseline$ on all measures. 
Only $\baseline$ and $\AS$ are able to produce policies that are statistically significantly different. Moreover, the cost distributions of $\AS$ are spaced more evenly with smaller dispersion and yield a tighter approximation of the expected Pareto-front. 
{This indicates the effectiveness of using dispersion to guide the sample placement in $\AS$. In $\baseline$, new samples are placed without such guidance, resulting in more \emph{unsuccessful} samples, \ie samples that end up being rejected as they are too similar to existing samples.}
These results highlight that $\AS$ is able to find better sets of MO-MRPD policies for different problem setups.

\begin{figure}[t]
    \centering   
        \includegraphics[width=0.9\linewidth]{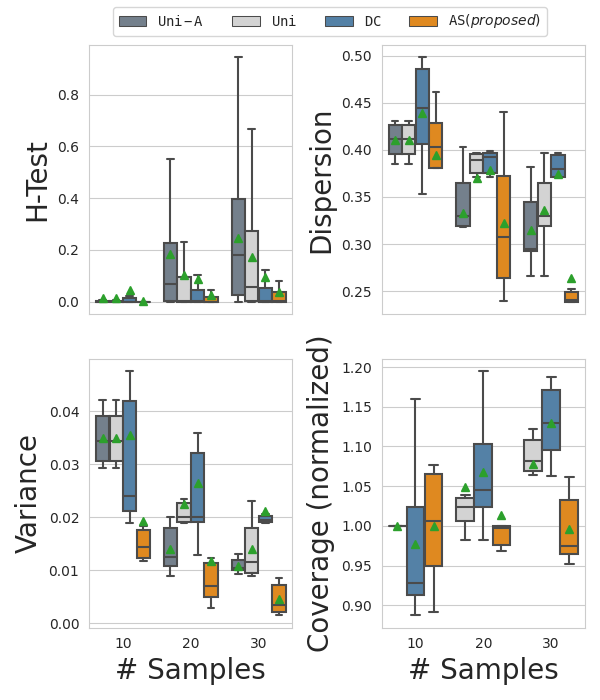}    
    \caption{Results for MRPD with three objectives. Coverage is normalized with respect to $\Uni$.}
    \label{fig:stats_3d}
\end{figure}

\paragraph{Results for three objectives}
We rerun the experiment with the total distance as a third objective and increase the sampling budget to $K\in\{10,20,30\}$.
The quantitative results are shown Figure \ref{fig:stats_3d}.

The outcome of the $\mathtt{Hypothesis\_Error}$ is mostly comparable to the experiment with two objectives. The uniform approaches still show an increase for larger $K$ - albeit smaller compared to two objectives --  such that their mean values still approach $.2$ while the upper end of the distributions exceed $.8$ and $.6$, respectively. In contrast, $\AS$ and the baseline $\baseline$ show only a small increase in error for larger $K$, with a minor advantage for $\AS$.

With respect to $\mathtt{Dispersion}$ our proposed method $\AS$ exhibits a strong advantage over all baselines as $K$ increases, even outperforming $\Uni$ which uses more samples.
In particular, $\Uni$ and $\baseline$ show only a minor to no improvement when increasing the budget from $K=20$ to $K=30$. In contrast, $\AS$ is able to further reduce dispersion significantly.
Similarly, on the $\mathtt{Variance}$ measure $\AS$ achieves the lowest value for all $K$, with a continued improvement as $K$ increases.
Lastly, the results for $\mathtt{Coverage}$ show a strong advantage of $\AS$: For all budgets $K$, the normalized variance has a mean of close to $1$, indicating that it is similar to $\Uni$. In contrast, the values for $\Unib$ and $\baseline$ increase for larger $K$.

In conclusion, the proposed method $\AS$ shows a strong performance on all measures. Indeed, the difference to the baselines is substantially larger than for the experiments with two objectives. Thus, the proposed approach is able to place samples efficiently in higher dimensions, allowing it to produce better sets of MO-MRPD policies.

\paragraph{Runtime}
{We briefly report average runtimes for the MRPD computation. Solving a single instance takes $\approx 3.9s$ for two objectives and $\approx 5.0s$ for three objectives. Thus, approximating a Pareto-front with $K=10$ samples using $\eta=20$ training instances takes approximately $780$ to $1,000s$. 
Yet, for larger instances such as city scale ride-pooling \cite{alonso2017demand}, the runtime of a single instance can be much larger, \textit{e.g.} up to one day.
The computation of the different training instances could be improved using parallelization.
}

\paragraph{Sensitivity to the number instances}

To generate statistically different policies, our method considers multiple $\eta$ randomly generated MRPD instances. 
To verify that the proposed algorithm performs well for different values of $\eta$, we investigate the sensitivity of the numerical results to the number of MRPD instances $\eta$.
Further, we validate that the number of test instances $\eta^{\mathtt{test}}$ used in the evaluation yields representative results. 
In the main experiments used $\eta =\eta^{\mathtt{test}}=20$, we now let $\eta$ and $\eta^{\mathtt{test}}$ take values in $\{10,20,30\}$.
We highlight the results for two experiment settings and run each with two and three features. The first setting is the lobby environment with $m=2$ robots and a sampling budget of $K=10$ and $K=20$ for two and three objectives, respectively. Under this setting we observed the largest deviations from the results reported in the previous sections. The second setting is the lobby environment with $m=8$ robots and sampling budget of $K=15$ and $K=30$ for two and three objectives, respectively.
Under this setting the H-test results of the proposed method ($\AS$) are among the highest values in the previous experiments (upper end of the boxplots in Figures \ref{fig:stats_2d} and \ref{fig:stats_3d}).

The results are summarized in Table \ref{tab:sensitivity}.
First, we consider changes in the number of training instances $\eta$, \textit{i.e.,} comparing columns in both tables.
We observe a difference between $\eta=10$ and $\eta=20$. Using too few training instances can lead to underestimating the variance of costs. Yet, our method still outperform $\Uni$ and $\Unib$ (see Figures \ref{fig:stats_2d}), and it can be expected that the results for $\baseline$ would be similarly affected by $\eta=10$.
Additionally, there is no significant difference between $\eta=20$ and $\eta=30$. The impact of $\eta$ is generally smaller in experiments with three objectives where all results remain under the tuning parameter $\Delta=.1$

Next, we consider the effect of varying the number of test instances $\eta^{\mathtt{test}}$ used in evaluation. Between the rows in both tables, we observe an increase in the H-test between $\eta^{\mathtt{test}}=10$ and $\eta^{\mathtt{test}}=20$. A larger number of test instances yields a larger statistical variance, which leads to higher chance of failing the H-test. Thus, a result for $\eta^{\mathtt{test}}=10$ suggests a low H-test result, while the error can be higher when considering a larger statistic. However, there is no significant increase between $\eta^{\mathtt{test}}=20$ and $\eta^{\mathtt{test}}=30$. Therefore, we conclude that previously reported results for $\eta^{\mathtt{test}}=20$ are reliable. 

{Lastly, we also consider the sensitivity towards the random sampling of tasks for a fixed $\eta$. That is, we use our main settings $\eta=20$ and $\eta^{\mathtt{test}}=20$ and repeat experiments with $10$ different random seeds. 
Overall, we found that the randomization of the seeds has only very little impact on the results: On all four measures reported in Figures \ref{fig:stats_2d} and \ref{fig:stats_3d}, the standard deviation over different seeds is $\approx.01$, and below $.001$ for the $\mathtt{Variance}$ measure. Thus, the reported results are robust to statistical differences when sampling task sequences in Algorithm \ref{alg:Pareto_Explore}.
}

Overall, the sensitivity analysis shows that the proposed algorithm performs well under different settings for $\eta$: For a lower value of $\eta=10$ $\AS$ still outperforms the baselines, while increasing $\eta$ to $30$ does not affect the results reported in the main experiments. Moreover, the evaluation results for $\eta^{\mathtt{test}}=20$  are reliable, \textit{i.e.,} a larger number of test instances does not change the outcome. {Finally, the sampling of training instances has only minimal impact on the algorithm performance.}\\

\setlength{\tabcolsep}{4.5pt}
\begin{table}[!t]

\begin{subtable}{.49\textwidth}
    \centering
    \begin{tabular}{l ccc ccc}
        \toprule
         & \multicolumn{3}{c}{2 Objectives} & \multicolumn{3}{c}{3 Objectives} 
        \\
       $\eta^{\mathtt{test}}$\hspace{-2.2cm}&
        $\eta=10$ & $\eta=20$ & $\eta=30$
       &$\eta=10$&$\eta=20$&$\eta=30$  \\
        \midrule

        $10$ &   
        .04 &.00 &.00 &
        .05  &.01 &.02  \\

        $20$ &   
        .14  &\textbf{.00 }&.00 &
        .09 &\textbf{.02 }&.04  \\

        $30$ &   
        .15 &.00&.00  &
        .08 &.03&.04 \\
        \bottomrule
    \end{tabular}
    \caption{Results for the H-test in the lobby environment with $m=2$ robots and budgets $k=10$ (two objectives) and $k=20$ (three objectives).} 
    \label{tab:lobby}
    \end{subtable}
    
    \bigskip
    \begin{subtable}{.49\textwidth}
    \centering
    \begin{tabular}{l ccc ccc}
        \toprule
         & \multicolumn{3}{c}{2 Objectives} & \multicolumn{3}{c}{3 Objectives} 
        \\
       $\eta^{\mathtt{test}}$\hspace{-.2cm}&
        $\eta=10$ & $\eta=20$ & $\eta=30$
       &$\eta=10$&$\eta=20$&$\eta=30$  \\
        \midrule

        $10$ &   
        .11 &.08 &.08 &
        .05 &.05  &.05 \\

        $20$ &   
        .17 &\textbf{.12} &.12 & 
        .04  &\textbf{.04 }&.05   \\

        $30$ &   
        .16 &.11 &.11 &
        .04 &.04  &.05  \\
        \bottomrule
    \end{tabular}
    \caption{Results for the H-test in the lobby environment with $m=8$ robots and budgets $k=15$ (two objectives) and $k=30$ (three objectives).}
    \label{tab:office}
    \end{subtable}\\
    \caption{Sensitivity of results for $\AS$ to the number of instances for different experiments. Shown are mean H-Test values when using different numbers of training instances $\eta$ and different numbers of test instances $\eta^{\mathtt{test}}$. Bold entries highlight the settings from the main experiments.}
    \label{tab:sensitivity}
    
\end{table}
\paragraph{Summary}
In conclusion, the numerical results show that the proposed algorithm $\AS$ computes sets of MRPD policies that outperform the baselines on several metrics for different MRPD settings and sampling budgets.
The H-test shows that policies found by the proposed method are statistically more distinguishable than policies found by uniform sampling.
While the baseline $\baseline$ is also able to produce statistically distinguishable policies, the dispersion, variance and coverage measures show that $\AS$ produces policies that better approximate the expected Pareto-front.
Finally, we verified that the algorithm is robust to changes in the number of training instances $\eta$, and the quantitative results are reliable.
%

\section{Discussion and Future Work}

We studied the problem of multi-objective MRPD where we want to find a set of policies that lead to different optimal trade-offs between given objectives.
A key feature of the problem is considering the statistics of the different objective values, caused by the stochastic nature of online task arrivals. Our problem formulation does not only seek to find different MRPD policies that approximate the expected Pareto-front, but also requires the policies to attain statistically different objective values.

By means of linear scalarization we converted the problem into one of finding a set of weights that balance the cost functions. We proposed an adaptive sampling method ($\AS$) and proved its completeness. 
Further, we presented how a state-of-the-art MRPD algorithm can be adapted to optimize for weighted cost functions for commonly used objective functions.
In simulation experiments, we demonstrated that $\AS$ is able to produce sets of high quality MRPD plans, outperforming several baseline approaches. Thus, the proposed framework provides system operators with a variety of options for configuring the robot behaviour to their preferences. 

{While we specifically focused on MO-MRPD, the proposed algorithm in Section \ref{sec:approach} can be applied to multi-objective formulations of a wider range of problems, such as multi-robot task assignment (MRTA) and multi-agent path finding (MAPF). Indeed, our approach does not make restrictive assumptions about the underlying MRPD solver. Further, the challenge of optimizing for competing objectives as well as stochastic problem inputs (\textit{e.g.,} requests or goals) are prevalent in many applications.}

%
One limitation of the proposed method is that it relies on linear scalarization of the multi-objective problem.
The main shortcoming of linear scalarization is that it is not \emph{Pareto-complete}: While every solution to the linear scalarization of the multi-objective problem is Pareto-optimal, there might exist Pareto-optimal solutions that are not a solution to \eqref{eq:MRPD_LSMOOP}. Thus, future work should consider other forms of scalarization to approach the MO-MRPD. One such method is using a weighted maximum instead of a weighted sum, also referred to as a Chebyshev scalarization. However, this poses major challenges for solving the MRPD problem given a choice of scalarization weights, since existing MRPD solvers are not able to optimize for such cost functions.

Future work should also consider HRI frameworks that help users to select the policy that best fits their preferences.
One approach could be choice-based learning where the user iteratively chooses between two presented options \cite{sadigh2017active, biyik2019asking, wilde2020improving, Wilde2021LearningSubmod}. Adapting this to MO-MRPD should explore how the variance of MRPD policies affect human choices, \textit{i.e.,} how humans can choose between different policies when their cost distributions are similar.
Such a framework would complement the presented theoretical work on exploring different MO-MRPD policies and thus further help system operators to adapt MRPD systems to their specific requirements.

Finally, our work focused specifically on online MRPD. However, the problem of finding trade-offs between competing objectives under stochastic demands is relevant in other variants of multi-robot task assignment (MRTA) and dynamic vehicle routing. Thus, future work could study how the proposed framework could be applied in a broader range of planning problems: {This could include considering inter-agent collisions as formalized in Multi-Agent Path Finding (MAPF), or other deployment problems} such as multi-robot informative path planning and team orienteering. 

\bibliographystyle{IEEEtran}

\newpage
\section{Biography Section}

\begin{IEEEbiography}[{\includegraphics[width=1in,height=1.25in,clip,keepaspectratio]{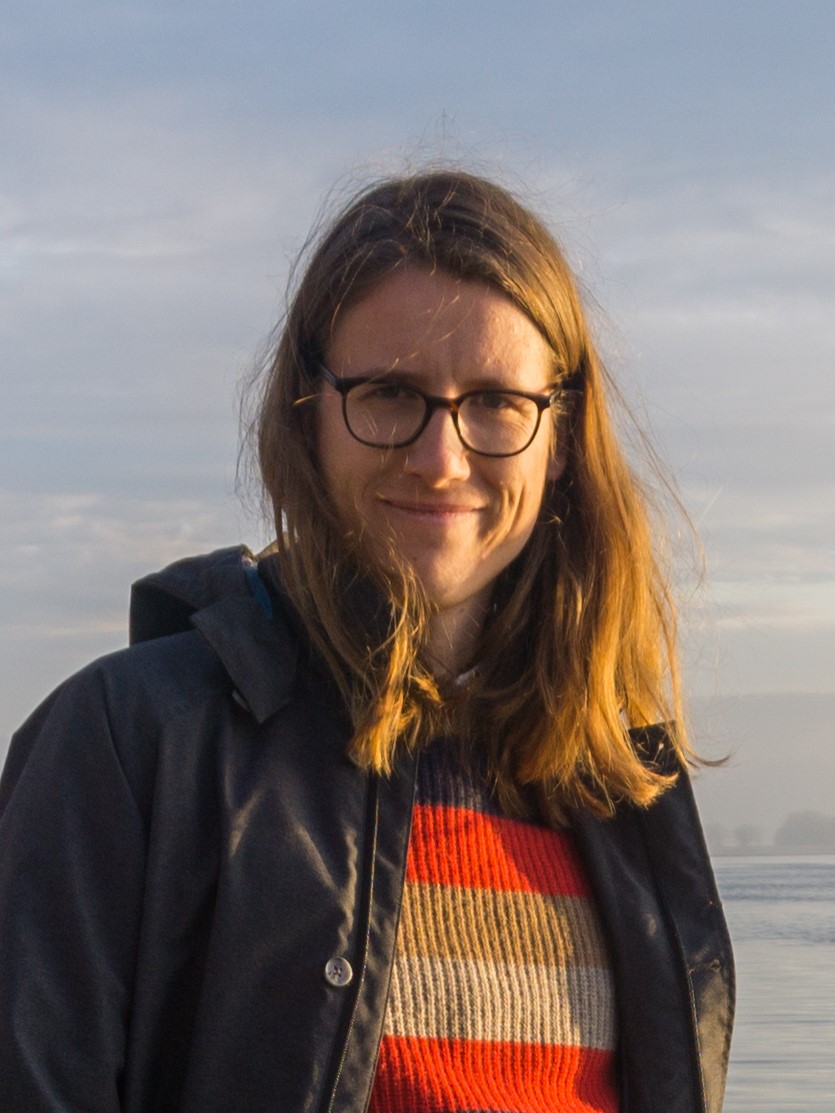}}]{Nils Wilde}
(Member, IEEE) is currently a Postdoctoral Fellow in the Autonomous Multi-Robots Lab working with Javier Alonso-Mora at TU Delft. Until August 2021 he was a postdoctoral fellow at the Autonomous Systems Lab at the University of Waterloo where he also did my PhD in Electrical and Computer Engineering (ECE) under the co-supervision of Dana Kulić and Stephen L. Smith from 2016 to 2020.

His research combines robot motion planning and human robot interaction (HRI), investigating how inexperienced users can define complex behaviours for autonomous mobile robots via active learning frameworks. Recent work broadens the focus to high level coordination of multi-robot systems under uncertainty as well as theoretical work on multi-objective optimization for robot planning problems. 
\end{IEEEbiography}
\vspace{11pt}
\begin{IEEEbiography}[{\includegraphics[width=1in,height=1.25in,clip,keepaspectratio]{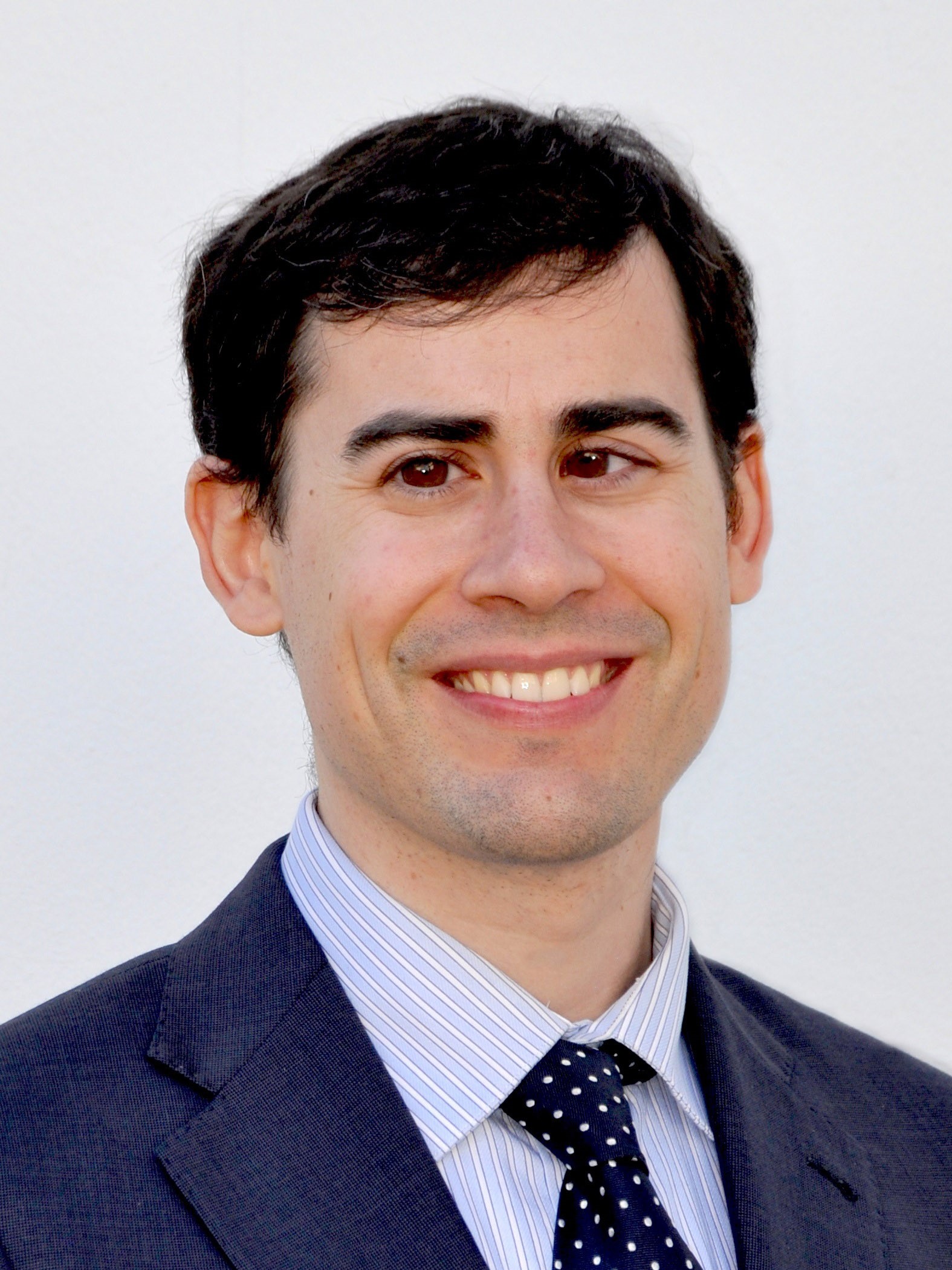}}]{Javier Alonso-Mora}
(Senior Member, IEEE)
received the M.Sc. and Ph.D. degrees in robotics
from ETH Zürich, Zürich, Switzerland, in 2010 and
2014, respectively.

He is currently an Associate Professor with
the Delft University of Technology, Delft,
The Netherlands. Until October 2016, he was
a Post-Doctoral Associate with the Computer
Science and Artificial Intelligence Laboratory,
Massachusetts Institute of Technology, Cambridge,
MA, USA. He was also a member of the Disney
Research, Zürich. His main research interests include autonomous navigation
of mobile robots, with a special emphasis in multi-robot systems and robots
that interact with other robots and humans. Toward the smart cities of the
future, he applies these techniques in various fields, including self-driving
cars, automated factories, aerial vehicles, and intelligent transportation
systems.

Dr. Alonso-Mora was a recipient of the European Research Council (ERC)
Starting Grant in 2021, the IEEE International Conference on Robotics and
Automation (ICRA) Best Paper Award on Multi-Robot Systems in 2019 and
the Veni Grant from the Netherlands Organization for Scientific Research
in 2017.

\end{IEEEbiography}

\end{document}